\documentclass[a4paper,final,reqno]{amsart}

\usepackage{algorithm}
\usepackage{algpseudocode}
\usepackage{amsmath, amssymb, amsfonts,amsthm}
\usepackage{array}
\usepackage{bbm}
\usepackage{bm}
\usepackage{booktabs}
\usepackage{color}
\usepackage{comment}
\usepackage{enumerate}
\usepackage{enumitem}
\usepackage{mathtools,microtype,times}
\usepackage{environ}
\usepackage{epsfig}
\usepackage[mathscr]{euscript}
\usepackage{float}
\usepackage{graphicx}
\usepackage{hyperref}
  \hypersetup{  
    bookmarksnumbered
  }
  \hypersetup{  
    breaklinks=false,
    pdfborderstyle={/S/U/W 0},
    citebordercolor=.235 .702 .443,
    urlbordercolor=.255 .412 .882,
    linkbordercolor=.804 .149 .149,
  }
\hypersetup{ 
    pdfauthor={},
    pdftitle={},
    pdfkeywords={}
  }
\hypersetup{colorlinks=true,linkcolor=blue,citecolor=magenta}
\usepackage{lineno}
\usepackage{mathtools}
\usepackage{nicefrac}
\usepackage{subcaption}
\usepackage{thmtools}
\usepackage{times}
\usepackage[draft]{todonotes} 
\usepackage{url}
\usepackage{xcolor}

  \usetikzlibrary{matrix,positioning,calc,fit,shapes.geometric,patterns}
  \pgfdeclarelayer{background}
  \pgfdeclarelayer{foreground}
  \pgfsetlayers{background,main,foreground}
  \tikzstyle{vec}=[circle,inner sep=1pt,outer sep=-1pt,fill]
  \tikzstyle{border}=[thick]
  \tikzstyle{favborder}=[border,dotted]
  \tikzstyle{exclborder}=[border,dashed]

\usepackage{scalerel}

\newcommand\reallywidehat[1]{\arraycolsep=0pt\relax%
\begin{array}{c}
\stretchto{
  \scaleto{
    \scalerel*[\widthof{\ensuremath{#1}}]{\kern-.5pt\bigwedge\kern-.5pt}
    {\rule[-\textheight/2]{1ex}{\textheight}} 
  }{\textheight} %
}{0.5ex}\\           
#1\\                 
\rule{-1ex}{0ex}
\end{array}
}

\newcommand{\pspace}{\varOmega}
\newcommand{\pr}{P}
\newcommand{\lpr}{{\underline{\pr}}}
\newcommand{\upr}{{\overline{\pr}}}
\newcommand{\apr}{Q}
\newcommand{\gambles}{\mathcal{L}}

\newcommand{\desirs}{\mathcal{D}}
\newcommand{\edesirs}{\mathcal{E}}
\newcommand{\reals}{\mathbb{R}}
\newcommand{\nats}{\mathbb{N}}
\DeclareMathOperator{\posi}{posi}
\newcommand{\partit}{\mathcal{B}}
\newcommand{\solp}{\mathcal{M}}
\newcommand{\eps}{\varepsilon}
\newcommand{\values}{\mathcal{X}}

\DeclareMathOperator{\median}{median}

\newtheorem{remark}{Remark}

\newtheorem{definition}{Definition}
\newtheorem{proposition}{Proposition}
\newtheorem{example}{Example}

\begin{document}

\title{Nonlinear desirability theory}

\author{Enrique Miranda} 
\address{University of Oviedo (Spain), Dep. of Statistics and Operations Research} 
\email{mirandaenrique@uniovi.es}
\author{Marco Zaffalon}
\address{IDSIA (Switzerland)}
\email{zaffalon@idsia.ch}

\begin{abstract}
Desirability can be understood as an extension of Anscombe and Aumann's Bayesian decision theory to sets of expected utilities. At the core of desirability lies an assumption of linearity of the scale in which rewards are measured. It is a traditional assumption used to derive the expected utility model, which clashes with a general representation of rational decision making, though. Allais has, in particular, pointed this out in 1953 with his famous paradox. We note that the utility scale plays the role of a closure operator when we regard desirability as a logical theory. This observation enables us to extend desirability to the nonlinear case by letting the utility scale be represented via a general closure operator. The new theory directly expresses rewards in actual nonlinear currency (money), much in Savage's spirit, while arguably weakening the founding assumptions to a minimum. We characterise the main properties of the new theory both from the perspective of sets of gambles and of their lower and upper prices (previsions). We show how Allais paradox finds a solution in the new theory, and discuss the role of sets of probabilities in the theory.
\end{abstract}

\maketitle

\section{Introduction}

The standard paradigm within decision making under uncertainty is based on the expected utility model: given a set of alternatives, we should select the one(s) maximising our utility, given the probabilities of the different outcomes. 

The foundations for this paradigm can be traced back to Savage's \cite{savage1972} and Anscombe and Aumann's work \cite{anscombe1963}. Both their axiomatisations aimed at modelling situations of nonlinear utility; yet both critically depend on some linearity assumption about the scale in which utilities are measured. This limits the scope of the expected utility model in a way that Allais pointed out long ago with his famous paradox \cite{allais1953}.

The paradox sparked a great deal of research that has attempted to remedy the shortcomings of expected utility, as for instance Machina's \cite{machina1982}, prospect theory \cite{kahneman1979,tversky1992}, rank-dependent expected utility \cite{quiggin1993} or regret theory \cite{bell1982,loomes1982}. These approaches have different degrees of generality and depart more or less radically from the original paradigm.

In this paper we would like to go back to the foundations of expected utility, and re-start from there with what we believe is the minimal change we need to make to them so as to enable a faithful and general treatment of nonlinearity. Key to our aim is desirability theory as proposed by Williams \cite{williams1975} and popularised by Walley \cite{walley1991} (see \cite{quaeghebeur2014} for an introduction). 

Desirability, or the theory of coherent sets of desirable gambles, originated as a generalisation of de Finetti's theory of probability \cite{finetti19745} to deal with imprecise information. It is a very general theory of uncertainty that encompasses as particular cases non-additive measures, such as possibility measures \cite{dubois1988}, belief functions \cite{shafer1976}, p-boxes \cite{ferson2003}, and sets of probabilities (robust Bayesian models). Somewhat unexpectedly, we showed in \cite{zaffalon2017a,zaffalon2018a} that desirability also (essentially) coincides with Anscombe and Aumann's theory \cite{anscombe1963} once this is generalised to handle sets of probabilities and utilities (at that point, a number of authors had already `robustified' expected utility so as to handle imprecision in both probability and utility, see, e.g., \cite{galaabaatar2013,nau2006,seidenfeld1995}). This created a very general theory of decision making but at the same time it made apparent that Allais paradox was still lurking in the background, as an undesired inheritance of Anscombe and Aumann's original work.

It made also clear precisely where the problem originates from, though, thus hinting at a way to correct for it: i.e., replacing the linear definition of the utility scale, which is `hard-wired' in all the mentioned approaches to expected utility, with a general, nonlinear, one. This is what we set out to do in this paper. 

The resulting theory of rational decision making is founded on three simple axioms that can intuitively be described as follows: gaining money is desirable; losing money is undesirable; the value of money is measured on a logically consistent utility scale (we use the term `money' in a very broad sense to mean amounts of goods under consideration, with no restriction on their cardinality). Note that we take rewards to be paid in money. In doing so, we get closer in spirit to Savage's approach, which assumes gambles to pay rewards in actual currency, unlike Anscombe and Aumann's that pays rewards indirectly via compound lotteries.

Technically, we shall go after our endeavour by making use of closure operators, which will determine those sets of desirable gambles that are internally consistent, and which shall then encompass our nonlinear dispositions towards rewards. We shall give the main notions of our theory in Section~\ref{sec:scale}, after some technical preliminaries. 

In the standard case of a linear utility scale, a set of desirable gambles allows us to determine a lower and an upper prevision, which may be given a behavioural interpretation as acceptable buying and selling prices, thus extending the work by de Finetti \cite{finetti19745} to the imprecise case. In Section~\ref{sec:lpr} we shall study what happens to this correspondence in the nonlinear case. The connection with sets of probabilities is analysed in Section~\ref{sec:credal}. In Section~\ref{sec:allais} we look at the preference relations that are encompassed by a closure operator and show that they can be used to give a solution to Allais paradox. We conclude the paper in Section~\ref{sec:conclusions} with some additional remarks. 

\section{Desirability with linear and nonlinear utility scales}\label{sec:scale}

Consider a possibility space $\pspace$. A \emph{gamble} $f:\pspace\rightarrow\reals$ is a bounded real-valued function on $\pspace$. For any two gambles $f,g$, we use $f\gneq g$ to denote that $f\geq g$ and $f\neq g$. We denote by $\gambles(\pspace)$ the set of all the gambles on $\pspace$ and by $\gambles^+(\pspace)\coloneqq\{f\in\gambles(\pspace): f\gneq0\}$ the subset of the \emph{positive gambles}. We denote these sets also by $\gambles$ and $\gambles^+$, respectively, when there is no ambiguity about the space involved. Negative gambles are defined by $\gambles^-\coloneqq-\gambles^+$, and we shall also use $\gambles^-_0\coloneqq\gambles^-\cup\{0\}$ and $\gambles_{<}:=\{f: \sup f<0\}$. Events are denoted by capital letters such as $A,B,C\subseteq\pspace$. We shall identify events with indicator functions, whence disjunctions ($A\cap B$) will be represented by products ($AB$). As a consequence, the product $Bf$ is equal to $f$ on $B$ and zero elsewhere. It is interpreted as a conditional gamble: one that is called off if $B$ does not occur. 
Finally, given a partition $\partit$ of $\pspace$, a gamble is said to be $\partit$-\emph{measurable} when it is constant on the elements of $\partit$; we shall denote by $\gambles_\partit$ the set of $\partit$-measurable gambles. 

The traditional approach to coherence in Williams-Walley's theory assumes that the scale in which the rewards, represented by gambles, are measured is linear \cite[Sect.~2.2]{walley1991}. This implies that the gambles whose desirability is implied by those from a given set $\desirs$ are those in its conic hull: 

\begin{definition}[\bf{Conic hull}] \label{def:posi}Given a set $\desirs\subseteq\gambles(\pspace)$, let 
\begin{equation*}
\posi(\desirs)\coloneqq\left\{ \sum_{j=1}^{r} \lambda_{j}f_{j}: f_{j} \in \desirs, \lambda_{j} > 0, r \ge 1\right\}
\end{equation*}
denote the \emph{conic hull} of the original set.
\end{definition}

Then a set of desirable gambles $\desirs\subseteq\gambles$ is called (Williams-)\emph{coherent} if and only if the following conditions hold:
\begin{enumerate}[label=\upshape D\arabic*.,ref=\upshape D\arabic*]
\item\label{D1} $\gambles^+ \subseteq \desirs$ [Accepting Partial Gains];
\item\label{D2} $0\notin \desirs$ [Avoiding Status Quo];
\item\label{D3} $f,g \in \desirs \Rightarrow f+g \in \desirs$ [Additivity];
\item\label{D4} $f \in \desirs, \lambda>0 \Rightarrow \lambda f \in \desirs$ [Positive Homogeneity].
\end{enumerate}
This is equivalent to requiring that $\posi(\desirs\cup \gambles^+)=\desirs$ and $\desirs\cap\gambles^-_0=\emptyset$. When we regard the theory of desirability from a logical perspective, $\posi$ corresponds to the deductive closure (this was pointed out by de Cooman long ago \cite{cooman2005e}); \ref{D1} to the tautologies and \ref{D2} to the status quo---which combined with the other axioms defines the contradictions, i.e., $\gambles^-$. For a deeper account of \emph{desirability}, we refer to \cite{couso2011,cooman2012b,miranda2010c,quaeghebeur2014} and \cite[Sect.~3.7]{walley1991}.

We proceed to generalise desirability by retaining the tautologies and the contradictions while replacing $\posi$ with the standard definition of a closure operator:

\begin{definition}[{\bf Closure operator}]\label{def:clo-ope} Let $\mathcal{P}(\cdot)$ denote power set. A map $\kappa:\mathcal{P}(\gambles)\rightarrow\mathcal{P}(\gambles)$ is a \emph{closure operator} if and only if for any two sets $\desirs,\desirs'\subseteq\gambles$ it satisfies:
\begin{enumerate}[label=\upshape C\arabic*.,ref=\upshape C\arabic*]
\item\label{C1} $\desirs\subseteq\kappa(\desirs)$ \emph{[Extensiveness]};
\item\label{C2} $\desirs\subseteq\desirs'\Rightarrow\kappa(\desirs)\subseteq\kappa(\desirs')$ \emph{[Monotonicity]};
\item\label{C3} $\kappa(\kappa(\desirs))=\kappa(\desirs)$ \emph{[Idempotency]}.
\end{enumerate}
We shall denote by $\mathcal{K}$ the family of all closure operators.
\end{definition}

Let us stop a moment to reflect on such a conceptual step. In the traditional case, the linear utility scale represented by $\posi$ prescribes how gambles relate to one another in terms of desirability. With general closure operators we now allow for such a relation to take on very diverse, and in particular nonlinear, forms. Therefore closure operators not only replace the linear utility scale in our approach, but could be said to be new forms of utility scales, which are nonlinear. Let us stress, to avoid confusion, that utility and utility scale are different concepts: in our approach, for example, we have utility scales (closure operators) but we do not have a utility in general; adopting closure operators implies that considerations of uncertainty and value are intertwined in a way that prevents them from being disentangled in general. Still, it is possible to relate the two concepts in special cases, as in the next example.

\begin{example}[Probability-utility pairs]\label{ex:prob-utility}
As a motivating example, let us show that the preferences encoded by a probability-utility pair can be incorporated into our formalism. To keep things simple, let $u:\reals\rightarrow\reals$ be a monotone and invertible utility function satisfying $u(0)=0$. Given $\lambda>0,f,g\in\gambles$, let us define the following operations: 
\begin{align*}
f\oplus g&\coloneqq u^{-1}(u(f)+u(g)),\\
\lambda\odot f&\coloneqq u^{-1}(\lambda u(f)),
\end{align*}
where $u:\gambles(\pspace)\rightarrow\gambles(\pspace)$ is defined by $(u\circ f)(\omega):=u(f(\omega))$ for all $\omega\in\pspace,f\in\gambles(\pspace)$.

Consider now the operator $\kappa$ given by 
\begin{align*}
\kappa(\desirs)&\coloneqq\gambles^+\cup\left\{f\in\gambles: f\ge\bigoplus_{i=1}^n \lambda_i\odot f_i\text{ for some } n\ge1, \lambda_i>0, f_i\in\desirs\right\}\\
&=\left\{f\in\gambles: f=\bigoplus_{i=1}^n \lambda_i\odot f_i\text{ for some } n\ge1, \lambda_i>0, f_i\in\desirs\cup\gambles^+\right\}.
\end{align*}
It is not difficult to check that $\kappa$ satisfies~\ref{C1}--\ref{C3}. 
\end{example}

We shall use a few examples of closure operators to illustrate the different notions we shall introduce throughout. 

\begin{example}\label{ex:first-example}
The following operators satisfy axioms~\ref{C1}--\ref{C3}: 
\begin{itemize}
 \item $\kappa_1(\desirs)\coloneqq\posi(\desirs).$
 \item $\kappa_2(\desirs)\coloneqq\{\sum_{i=1}^{n} f_i: n\in\nats, f_1,\dots,f_n\in\desirs\}$.
 \item $\kappa_3(\desirs)\coloneqq\{g\geq \lambda f: f\in\desirs, \lambda>0\}$.
 \item $\kappa_4(\desirs)\coloneqq\{g\geq f:f\in\desirs\}$.
\end{itemize}
Out of these, $\kappa_1$ is the linear closure operator used in traditional (Williams-)coherence; $\kappa_4$ is the closure operator associated with $2$-convexity \cite[Sect.~6.2]{pelessoni2016}; while $\kappa_3$ is related, but not equivalent, to $2$-coherence, as showed also in \cite[Sect.~6.1]{pelessoni2016}.\footnote{The difference is that the axiomatisation of $2$-coherence in terms of desirability requires in addition that the sum of two desirable gambles must have a positive supremum.} 
\end{example}

\begin{remark}
In past work \cite{zaffalon2017a,zaffalon2018a}, we generalised desirability to utility considerations by introducing a set of prizes $\values$ besides a possibility space $\pspace$. Gambles were defined on the product space $\pspace\times\values$; coherence was kept standard via \ref{D1}--\ref{D4}. As mentioned already, this amounts to generalising Anscombe and Aumann's work \cite{anscombe1963} to sets of expected utilities and, under proper conditions, to obtain a utility function over prizes that is generally nonlinear. Nonetheless, the underlying machinery is still linear, due to \ref{D3} and \ref{D4} (that is,  due to $\kappa_1$), and for this reason it still incurs problems such as Allais paradox. In the current work, we directly target such a basic issue by replacing $\kappa_1$ with any closure operator, thus giving up on linearity altogether. Gambles are defined only on $\pspace$ and their values can naturally be interpreted as amounts of money, as in the tradition of Savage. 

On the other hand, the idea of nonlinear utility in the context of sets of desirable gambles has also been considered in the recent work in \cite{casanova2023a}. There are a few differences with the work we are carrying out here, though: on the one hand, the authors consider the implications of a \emph{finite} set of assessments of desirability, while we consider the implications of an arbitrary family $\desirs$; they focus on a few examples of axiomatisations that are weaker than traditional desirability, and show that they can be formulated as a classification problem, while here we work more generally with an arbitrary closure operator $\kappa$; and they also allow for assessments of rejection (non-desirability), which is something we do not take into account in our model. $\lozenge$
\end{remark}

Let us present how to analyse the consistency of a desirability assessment $\desirs$ with respect to some fixed closure operator $\kappa$. Traditionally, this is done by requiring that:
\begin{itemize}
    \item[(i)] the implications of our desirability assessments do not make us subject to a sure loss; and 
    \item[(ii)] our set is deductively closed, in that it includes all the gambles whose desirability is implied by those in $\desirs$. 
\end{itemize}
In order to generalise these two ideas, we need first to give a proper definition of what the \emph{implications} of our assessments are. This is given by what we shall call the natural extension: 

\begin{definition}[{\bf Natural extension}] Given a set $\desirs$ of desirable gambles, its \emph{natural extension} with respect to a closure operator $\kappa$ is given by $\edesirs_\kappa(\desirs)\coloneqq\kappa(\desirs\cup\gambles^+)$.
\end{definition}

The idea behind the above notion is straightforward: the natural extension is given by the closure of the set of gambles whose desirability we have already assessed. Since we can assume without loss of generality that any positive gamble must be desirable, we must also include those in our set before applying $\kappa$. 

Once we have established the notion of natural extension, we can give the expression of the requirements (i) and (ii) mentioned above in terms of a general closure operator $\kappa$: 

\begin{definition}[{\bf Avoiding partial and sure loss for gambles}] A set $\desirs$ of desirable gambles is said to \emph{avoid partial loss} if and only if $\gambles^-_0\cap\edesirs_\kappa(\desirs)=\emptyset$. It is said to \emph{avoid sure loss} if and only if $\gambles_{<}\cap\edesirs_{\kappa}(\desirs)=\emptyset$. 
\end{definition}

The difference between the two notions lies in which gambles are considered to be undesirable: in the case of avoiding partial loss, we exclude all those gambles $f$ that can never give us a positive utility, no matter the outcome of the experiment; while in the weaker notion of avoiding sure loss we only exclude the gambles that make us always lose some positive amount of utiles. This means for instance that a set of gambles that includes the zero gamble may avoid sure loss but it will never avoid partial loss. 

Requirement (ii) leads to the general notion of coherence:
\begin{definition}[{\bf Coherence relative to a set of gambles}]\label{def:rkcoh}
Say that $\desirs$ is \emph{coherent relative to} a superset $\apr\subseteq\gambles$ if and only if $\desirs$ avoids partial loss and
$\apr\cap\edesirs_\kappa(\desirs)\subseteq\desirs$ (and hence
$\apr\cap\edesirs_\kappa(\desirs)=\desirs$).
\end{definition}

The reason why we are considering in the definition above coherence with respect to some set of gambles is that we shall later apply this notion after making operations of marginalisation or conditioning, which will restrict our framework to proper subsets of $\gambles$. In the particular case where $\apr=\gambles$, we shall simply say that $\desirs$ is \emph{coherent}, and the condition can be characterised in the following manner: 

\begin{proposition}\label{prop:kcoh} $\desirs$ is coherent if and only if it satisfies the following conditions:
\begin{enumerate}[label=\upshape K\arabic*.,ref=\upshape K\arabic*]
\item\label{K1} $\gambles^+ \subseteq \desirs$ \emph{[Accepting Partial Gains]};
\item\label{K2} $\gambles^- _0\cap\desirs=\emptyset$ \emph{[Avoiding Partial Loss]};
\item\label{K3}$\kappa(\desirs)=\desirs$ \emph{[Deductive Closure]}.
\end{enumerate}
\end{proposition}

\begin{proof}
Def.~\ref{def:rkcoh} with $\mathcal{Q}:=\gambles$ implies that $\desirs$ is coherent if and only if $\kappa(\desirs\cup \gambles^+)=\desirs$ and $\desirs\cap\gambles^-_0=\emptyset$. Let us prove that this is equivalent to conditions~\ref{K1}--\ref{K3}. 

Assume that $\desirs$ is coherent. Applying~\ref{C1} and~\ref{C2}, we deduce that 
\[\gambles^+\subseteq\kappa(\gambles^+)\subseteq\kappa(\gambles^+\cup\desirs)=\desirs,
\] whence~\ref{K1} holds. This implies that $\kappa(\desirs)=\kappa(\desirs\cup\gambles^+)=\desirs$, whence~\ref{K3} also holds. Finally,~\ref{K2} follows by definition. 

Conversely, if $\desirs$ satisfies~\ref{K1}--\ref{K3} then 
\[
\edesirs_{\kappa}(\desirs)=\kappa(\desirs\cup\gambles^+)=\kappa(\desirs)=\desirs,
\]
using~\ref{K1} and~\ref{K3} for the second and third equalities, respectively. Moreover, by~\ref{K2}, $\gambles^-_0\cap \edesirs_{\kappa}(\desirs)=\gambles^-_0\cap \desirs=\emptyset$, whence $\desirs$ also avoids partial loss. From this we conclude that $\desirs$ is coherent. 
\end{proof}

\begin{example}
Let us consider again the closure operator $\kappa$ from Example~\ref{ex:prob-utility}, that is associated with an invertible utility function $u$. If we now consider a linear prevision (that is, the expectation operator with respect to a finitely additive probability) $\pr$ on $\pspace$ and the set of gambles
\[
\desirs\coloneqq\gambles^+\cup\{f: \pr(u(f))>0\},
\]
it holds that $\kappa(\desirs)=\desirs$ and that $\desirs\cap\gambles^-_0=\emptyset$. In other words, $\desirs$ is $\kappa$-coherent. Moreover, it is easy to see that
\[
\kappa(\desirs)=u^{-1}(\kappa_1(u(\desirs))),
\]
showing even more clearly that linear utility can be recovered as a special case. 
\end{example}

If we compare axioms~\ref{K1}--\ref{K3} with axioms~\ref{D1}--\ref{D4}, which characterise Williams-cohe\-rence, we can see more clearly that the closure operator takes the role of axioms~\ref{D3} (additivity) and~\ref{D4} (positive homogeneity). The natural extension by a closure operator $\kappa$ plays a similar role as in traditional desirability theory: 

\begin{proposition}\label{prop:natex}
Consider a closure operator $\kappa$, and let $\desirs$ be a set of gambles. Then 
$\desirs$ avoids partial loss if and only if it has a coherent superset. In that case, $\edesirs_{\kappa}(\desirs)$ is the smallest coherent superset of $\desirs$.
\end{proposition}

\begin{proof}
First of all, note that
 \begin{equation*}
  \kappa(\edesirs_{\kappa}(\desirs))=\kappa(\kappa(\desirs\cup\gambles^+))=\kappa(\desirs\cup\gambles^+)=\edesirs_{\kappa}(\desirs),
 \end{equation*}
using \ref{C3} for the second equality. By \ref{C1}, $\gambles^+\subseteq \kappa(\desirs\cup\gambles^+)=\edesirs_{\kappa}(\desirs)$. Therefore, $\edesirs_{\kappa}(\desirs)$ is coherent if and only if it avoids partial loss, which by definition is equivalent to $\desirs$ avoiding partial loss. This shows that if $\desirs$ avoids partial loss then it has a coherent superset. 

Conversely, for any coherent superset $\desirs'$ of $\desirs$ it holds that 
\[
\desirs\cup\gambles^+ \subseteq \desirs'\Rightarrow \kappa(\desirs\cup\gambles^+) \subseteq \kappa(\desirs')=\desirs',
\]
whence $\edesirs_{\kappa}(\desirs)\subseteq\desirs'$. This implies that $\edesirs_{\kappa}(\desirs)\cap\gambles^-_0=\desirs'\cap\gambles^-_0=\emptyset$, whence $\edesirs_{\kappa}(\desirs)$ avoids partial loss and as a consequence $\desirs$ avoids partial loss. 

We conclude from the proof above that (a) if $\desirs$ avoids partial loss then $\edesirs_{\kappa}(\desirs)$ is a coherent superset; and (b) $\edesirs_{\kappa}(\desirs)$ is included in any other coherent superset of $\desirs$, and as a consequence it is the smallest one. 
\end{proof}

Given a closure operator $\kappa$ we shall denote by $\Lambda_{\kappa}$ the family of $\kappa$-coherent sets of gambles, i.e., those sets that are deductively closed with respect to $\edesirs_{\kappa}$. This leads to the following definition: 

\begin{definition}[{\bf Equivalence of closure operators}]\label{def:equiv}
Two operators $\kappa,\kappa'$ are called \emph{equivalent} if and only if $\Lambda_{\kappa}=\Lambda_{\kappa'}$; we say that $\kappa$ \emph{implies} $\kappa'$ if and only if $\Lambda_{\kappa}\subseteq\Lambda_{\kappa'}$, i.e., if and only if any set of gambles that is $\kappa$-coherent is also $\kappa'$-coherent. 
\end{definition}

It is immediate to prove that the above condition determines an equivalence relationship on the family of closure operators; we next show that it is not trivial, in the sense that two different closure operators $\kappa,\kappa'$ may be equivalent:  

\begin{example}\label{ex:equiv}
Consider $f\in\gambles_{<}$ and let us define the closure operators $\kappa$ and $\kappa'$ by: 
\[
\kappa(\desirs)\coloneqq\begin{cases}
 \gambles^+ &\text{ if } \desirs \subseteq\gambles^+ \\
 \gambles &\text{ otherwise}
 \end{cases}
\]
and 
\[
\kappa'(\desirs)\coloneqq\begin{cases}
 \gambles^+ &\text{ if } \desirs \subseteq\gambles^+ \\
 \kappa_4(\desirs\cup\{f\}) 
 &\text{ otherwise.}
 \end{cases}
\]
Then $\Lambda_{\kappa}=\Lambda_{\kappa'}=\{\gambles^+\}$, so they are equivalent; however, $\kappa(\{f\})=\gambles\neq\kappa'(\{f\})=\{g\geq f\}$.
\end{example}

If we go back to the closure operators in Ex.~\ref{ex:first-example}, it is easy to prove that $\Lambda_{\kappa_1}\subseteq\Lambda_{\kappa_2}\cap\Lambda_{\kappa_3}$ and that $\Lambda_{\kappa_2}\cup\Lambda_{\kappa_3}\subseteq\Lambda_{\kappa_4}$. 
Let us show that there is no additional inclusion relationship, and, as a consequence, that no two of these closure operators are equivalent: 

\begin{example}\label{ex:kappa-different}
Consider a binary space $\pspace$, and the following sets of desirable gambles: 
\begin{itemize}
 \item $\desirs_1\coloneqq\gambles^+\cup\{f\geq (-n,n) \text{ for some } n\in\mathbb{N}\}$. Then $\desirs_1$ is $\kappa_4$- and $\kappa_2$-coherent, but neither $\kappa_1$- nor $\kappa_3$-coherent. Thus, $\Lambda_{\kappa_1}\subsetneq\Lambda_{\kappa_2}$, $\Lambda_{\kappa_3}\subsetneq\Lambda_{\kappa_4}$ and $\Lambda_{\kappa_2}\nsubseteq\Lambda_{\kappa_3}$.  
 \item $\desirs_2\coloneqq\{f\geq (\lambda,-\lambda) \text{ for some } \lambda \neq 0\}$. Then $\desirs_1$ is $\kappa_4$- and $\kappa_3$-coherent, but neither $\kappa_1$- nor $\kappa_2$-coherent. Thus, $\Lambda_{\kappa_1}\subsetneq\Lambda_{\kappa_3}$, $\Lambda_{\kappa_2}\subsetneq\Lambda_{\kappa_4}$ and $\Lambda_{\kappa_3}\nsubseteq\Lambda_{\kappa_2}$.  
\end{itemize}  
\end{example}

These implications are summarised in Fig.~\ref{fig:implics}.
\begin{figure}[h]
\begin{center}
\begin{tikzpicture}
\draw (0,1) node(k1) {$\kappa_1$};
\draw (-1,0) node(k3) {$\kappa_2$};
\draw (1,0) node(k4) {$\kappa_3$};
\draw (0,-1) node(k2) {$\kappa_4$};
\draw[->] (k1) -- (k3);
\draw[->] (k1) -- (k4);
\draw[->] (k3) -- (k2);
\draw[->] (k4) -- (k2);
\end{tikzpicture}
\end{center}
\caption{Implications between the examples of closure operators.}\label{fig:implics}
\end{figure}
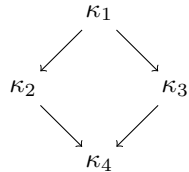

In this paper, we shall consider only closure operators that satisfy the following additional axiom:
\begin{definition}[{\bf Dominance}]\label{def:clo-dom} 
A set of gambles $\desirs$ is said to be \emph{closed under dominance} when $\desirs=\{g\geq f:f\in\desirs\}=\kappa_4(\desirs).$ A closure operator $\kappa$ \emph{satisfies dominance} if and only if
\begin{enumerate}[label=\upshape C4.,ref=\upshape C4]
\item\label{C4} $(\forall \desirs\in\Lambda_{\kappa})\ \desirs$ is closed under dominance \emph{[Dominance]}.
\end{enumerate}
We shall denote by $\mathcal{K}_d$ the subfamily of $\mathcal{K}$ given by those closure operators that are closed under dominance. 
\end{definition}
\noindent We regard Axiom~\ref{C4} as a very mild requirement once \ref{K1} is accepted. Indeed, if desire $g$ and $f\gneq g$, exchanging $g$ with $f$ means adding a positive gain, which is something that we should be disposed to do even if we do not impose the additivity axiom~\ref{D3} for arbitrary sums of desirable gambles. For this reason, in the remainder of this paper we shall focus only on closure operators that belong to $\mathcal{K}_d$, even if the majority of the notions and results can be extended for arbitrary elements of $\mathcal{K}$. Note that the smallest closure operator in $\mathcal{K}_d$ is $\kappa_4$. 

Closure operators satisfy the following:

\begin{proposition}\label{prop:arbitrary-intersections}
For any set of gambles $\apr\subseteq\gambles$ and for any closure operator $\kappa\in\mathcal{K}_d$, the family of $\kappa$-coherent sets of gambles relative to $\apr$ is closed under arbitrary intersections. 
\end{proposition}

\begin{proof}
Let $(\desirs_i)_{i\in I}$ be a family of $\kappa$-coherent sets of gambles relative to $\apr$, and let $\desirs\coloneqq\cap_{i \in I} \desirs_i$. Then 
\begin{equation*}
(\forall i\in I)\ \desirs\subseteq \desirs_i \Rightarrow \desirs \cup \gambles^+ \subseteq \desirs_i \cup \gambles^+ \Rightarrow\kappa(\desirs \cup \gambles^+)\subseteq \kappa(\desirs_i \cup \gambles^+),
\end{equation*}
where last implication follows by~\ref{C2}. This means that
\begin{equation*}
(\forall i \in I)\ \apr \cap \edesirs_{\kappa}(\desirs) \subseteq \apr \cap \edesirs_{\kappa}(\desirs_i)=\desirs_i \Rightarrow \apr \cap \edesirs_{\kappa}(\desirs) \subseteq \cap_{i\in I}\desirs_i=\desirs.
\end{equation*}
Therefore, $\desirs$ is coherent relative to $\apr$. 
\end{proof}

\begin{remark}
The present work has a natural relation with Casanova et al.'s work \cite{casanova2022b}. Such a work embeds traditional desirability into Kohlas' formalism of information algebras \cite{kohlas2003}. \emph{Information algebras} abstract the essential properties of a belief system, and the operations (namely, combination and extraction) needed to coherently aggregate different pieces of information and to do logical inference. The notion of `information order' relates instead to the fact that  $\mathbb{D}$ is a complete lattice. $\lozenge$
\end{remark}

Any closure operator $\kappa$ determines a family $\Lambda_{\kappa}$ of $\kappa$-coherent sets. Conversely, if we consider a family $\Lambda$ of subsets of $\gambles$, we may consider under which cases there is a closure operator $\kappa\in\mathcal{K}_d$ inducing it. This is determined by the following proposition: 

\begin{proposition}
Consider a family $\Lambda$ of subsets of $\gambles$, and let us define the 
operator
\begin{equation}\label{eq:closure-from-family}
\kappa(\desirs)\coloneqq\begin{cases}
\cap_{\desirs\subseteq\desirs'\in\Lambda} \desirs' &\text{ if } \{\desirs'\in\Lambda: \desirs\subseteq\desirs'\}\neq\emptyset;\\ 
\gambles &\text{ otherwise}.
\end{cases}
\end{equation}
\begin{itemize}
    \item[(a)] $\kappa$ satisfies \ref{C1}--\ref{C3}.
    \item[(b)] $\kappa$ satisfies~\ref{C4} if and only if any $\desirs'\in\Lambda$ is closed under dominance.
    \item[(c)] $\Lambda_{\kappa}=\Lambda$ if and only if any $\desirs'\in\Lambda$ is closed under dominance and satisfies $\gambles^+\subseteq\desirs', \desirs'\cap\gambles^-_0=\emptyset$ and $\Lambda$ is closed under arbitrary nonempty intersections.
\end{itemize}
\end{proposition}

\begin{proof}
\begin{itemize}
\item[(a)] First of all, for any set of gambles $\desirs$ it holds that  $\desirs\subseteq\kappa(\desirs)$ by construction, whence \ref{C1} holds. For \ref{C2}, note that $\desirs_1\subseteq\desirs_2 \Rightarrow\{\desirs' \in\Lambda: \desirs_2\subseteq\desirs'\}\subseteq\{\desirs'\in\Lambda: \desirs_1\subseteq\desirs'\}\Rightarrow\kappa(\desirs_1)\subseteq\kappa(\desirs_2)$.

Finally, for \ref{C3} note that $\kappa(\kappa(\desirs))\supseteq \kappa(\desirs)$ by \ref{C1} and \ref{C2}. The converse inclusion holds if and only if any $\desirs'\in\Lambda$ that includes $\desirs$ also includes $\kappa(\desirs)$; but this is a consequence of Eq.~\eqref{eq:closure-from-family}. 
    
\item[(b)] The direct implication holds because the intersection of sets that are closed under dominance is again closed under dominance, while the converse implication holds because $\kappa(\desirs')=\desirs'$ for any $\desirs'\in\Lambda$. 

\item[(c)] 
Let us prove that the conditions are necessary. On the one hand, for $k$ to belong to $\mathcal{K}_d$ it is necessary from (b) that any $\desirs\in\Lambda$ is closed under dominance. On the other hand, since $\kappa(\desirs)=\desirs$ for any $\desirs\in\Lambda=\Lambda_{\kappa}$, it must be  $\desirs=\edesirs_{\kappa}(\desirs)=\kappa(\desirs\cup\gambles^+)$, whence $\gambles^+\subseteq\desirs$ and also $\desirs\cap\gambles^-_0=\emptyset$. Finally, given a subfamily ${\mathcal H}\subseteq\Lambda$ and the set $\desirs\coloneqq \cap_{\desirs'\in{\mathcal H}} \desirs'$, it is 
    \[
    \kappa(\desirs)=\cap_{\desirs\subseteq\desirs'\in\Lambda} \desirs' \subseteq \cap_{\desirs\subseteq\desirs'\in\mathcal{H}} \desirs'=\desirs,
    \]
whence $\kappa(\desirs)=\desirs$. Moreover, $\gambles^+\subseteq\desirs$ and $\desirs\cap\gambles^-_0=\emptyset$, whence $\desirs\in\Lambda_{\kappa}=\Lambda$. Thus, $\Lambda$ must be closed under arbitrary intersections. 

To see that the conditions are sufficient, note that if $\gambles^+\subseteq\desirs\subseteq(\gambles^-_0)^c$ it follows that $\desirs\in\Lambda_{\kappa}$ for every $\desirs\in\Lambda$. On the other hand, if any arbitrary nonempty intersection of elements from $\Lambda$ belongs to $\Lambda$, we obtain in particular that $\kappa(\desirs)\in\Lambda$ whenever  $\{\desirs'\in\Lambda: \desirs\subseteq\desirs'\}\neq\emptyset$, and as a consequence that $\Lambda_{\kappa}\subseteq\Lambda$ by Prop~\ref{prop:kcoh}. Since by (b) if any element of $\Lambda$ is closed under dominance $\kappa\in\mathcal{K}_d$, we deduce that the conditions are sufficient. \qedhere
\end{itemize}
\end{proof}

Our next goal is to determine whether any coherent set can always be obtained as the conjunction of some subfamily of coherent supersets. 

\subsection{Strong belief structures}

In the case of linear utility, there exists a subfamily of $\Lambda_{\kappa}$ that represents maximally precise information, and moreover any coherent set can be obtained from the intersection of its supersets in this family, obtaining thus a `strong belief structure' (resp., a  `completely atomistic information algebra') in de Cooman's (resp., Kohlas') terminology. In this section, we analyse this property for general closure operators. 

As we shall show, in the more general setting we are considering in this paper we must distinguish between the notions of maximality and decisiveness.  

\begin{definition}[{\bf Maximality} and {\bf decisiveness}]
Consider $\kappa\in\mathcal{K}_d$. A $\kappa$-coherent set of gambles $\desirs$ is called: 
\begin{itemize}
 \item \emph{maximal} if and only if it has no $\kappa$-coherent superset;
 \item \emph{decisive} if and only if for every $f\neq 0$, exactly one of $f$ or $-f$ belongs to $\desirs$.
\end{itemize}
\end{definition}

For a given closure operator $\kappa$, the notion of maximality has a straightforward interpretation: it means that $\desirs$ is undominated in the partial order that can be established in the family of $\kappa$-coherent sets by means of set inclusion. The interpretation of decisiveness is instead that there is no indecision as to whether a gamble or its negation is desirable. This will more easily be understood when we analyse the properties of the lower and upper previsions associated with the set of desirable gambles $\desirs$ in the next section. We shall denote by $\overline{\Lambda}_{\kappa}$ and $\tilde{\Lambda}_{\kappa}$ the families of $\kappa$-maximal and $\kappa$-decisive sets, respectively. 

The relationship of implication between closure operators determines some inclusions between the decisive sets: 

\begin{proposition}\label{prop:inclusion-maximality}
Let $\kappa,\kappa'\in\mathcal{K}_d$ be such that $\kappa$ implies $\kappa'$. Then  
 $\tilde{\Lambda}_{\kappa} \subseteq \tilde{\Lambda}_{\kappa'}$.
\end{proposition}

\begin{proof}
If $\desirs\in\tilde{\Lambda}_{\kappa}$, then for any $f\neq 0$ either $f$ or $-f$ belongs to $\desirs$; considering also that $\desirs\in\Lambda_{\kappa}\subseteq\Lambda_{\kappa'}$, we conclude that $\desirs\in\tilde{\Lambda}_{\kappa'}$. 
\end{proof}

A few more observations are in order. First of all, in the case of linear utility scale, the two conditions above are equivalent \cite[Prop.~2]{cooman2012b}. The same applies to the closure operator $\kappa_2$: 

\begin{proposition}\label{prop:maxim-kappa4}
$\overline{\Lambda}_{\kappa_2}=\tilde{\Lambda}_{\kappa_2}$. 
\end{proposition}

\begin{proof}
We begin by showing that $\overline{\Lambda}_{\kappa_2}\subseteq\tilde{\Lambda}_{\kappa_2}$. Assume ex-absurdo the existence of $\desirs\in\overline{\Lambda}_{\kappa_2}\setminus\tilde{\Lambda}_{\kappa_2}$. This means that there is some $f\neq 0$ such that $f,-f\notin\desirs$ (note that it cannot be that both $f,-f\in\desirs$ because then their sum $f-f=0$ would also belong to $\desirs$, contradicting $\kappa_2$-coherence). 

Let us define $\desirs_1\coloneqq \desirs\cup\{f\}$. Then by \ref{C2} $\edesirs_{\kappa_2}(\desirs)=\desirs\subsetneq \edesirs_{\kappa_2}(\desirs_1)$. Since $\desirs$ is $\kappa_2$-maximal, this implies that $\edesirs_{\kappa_2}(\desirs_1)$ cannot be coherent. Therefore, there are gambles $f_1,\dots,f_n\in\desirs_1$ such that $\sum_{i=1}^{n}f_i=0$, the equality following from $\gambles^+\subseteq\desirs$ by coherence. At least one of these gambles must be equal to $f$, or we contradict the coherence of $\desirs$, and it cannot be all of them equal to $f$, because in that case we would have that $f=0$, also a contradiction. Let us assume without loss of generality that $f_1=\dots=f_{n_1}=f$ and $f_{n_1+1},\dots,f_n\in\desirs$ for some $1\leq n_1<n$. Then we deduce that $-n_1 f=\sum_{i=n_1+1}^{n} f_i \in\desirs$. 

A similar reasoning, starting with $\desirs_2\coloneqq\desirs\cup\{-f\}$ allows us to find $n_2\in\mathbb{N}$ such that $n_2 f\in \desirs$. But then 
\begin{equation*}
 0=n_2(-n_1 f)+ n_1 (n_2 f)\in \edesirs_{\kappa_2}(\desirs)=\desirs,
\end{equation*}
a contradiction with the coherence of $\desirs$. 

The inclusion $\tilde{\Lambda}_{\kappa_2}\subseteq\overline{\Lambda}_{\kappa_2}$ follows immediately from the definition of $\kappa_2$-maximality: if there was some $\desirs\in\tilde{\Lambda}_{\kappa_2}$ that was strictly included in some $\kappa_2$-coherent set $\desirs'$, then there should be some gamble $f\neq 0$ such that both $f$ and $-f$ belong to $\desirs'$, and as a consequence also $0=f-f$ would belong to $\desirs'$, a contradiction. 
\end{proof}

Note that, while $\Lambda_{\kappa_1}\neq\Lambda_{\kappa_2}$, in both cases the notions of maximality and decisiveness agree: $\overline{\Lambda}_{\kappa_1}=\tilde{\Lambda}_{\kappa_1}$ and $\overline{\Lambda}_{\kappa_2}=\tilde{\Lambda}_{\kappa_2}$. 
The equivalence between maximality and decisiveness does not hold for all closure operators: for instance $\overline{\Lambda}_{\kappa_4}=\{\gambles\setminus\gambles^-_0\}$ while $\tilde{\Lambda}_{\kappa_4}$ includes $\tilde{\Lambda}_{\kappa_1}$, which has more than one subset of $\gambles$. Hence, $\tilde{\Lambda}_{\kappa_4}$ strictly includes $\overline{\Lambda}_{\kappa_4}$. A similar comment applies to the closure operator $\kappa_3$. 

\begin{definition}[{\bf Strong belief structure}]
We say that the family $\Lambda_{\kappa}$ of $\kappa$-coherent sets of gambles forms a maximal  (resp., decisive) strong belief structure if and only if any $\kappa$-coherent set of gambles is the intersection of its maximal (resp., decisive) supersets. 
\end{definition}

It is well-known \cite[Sect.~2]{cooman2012b} that the family of $\kappa_1$-coherent sets forms a strong belief structure. This is not the case for the family of $\kappa_2$-coherent sets. 

\begin{example}
Consider a binary space $\pspace$, and let $\desirs\coloneqq\edesirs_{\kappa_2}(\{(-2,2)\})$. By construction, it is given by 
\[
\desirs=\gambles^+\cup\{f\geq (-2n,2n): n\geq 1\}.
\]
As a consequence, given $f\coloneqq(1,-1)$ it holds that neither $f$ nor $-f$ belongs to $\desirs$. Moreover, there is no coherent superset of $\desirs$, let alone a maximal one, that includes $f$: it there was, it should include $\edesirs_{\kappa_2}(\desirs\cup\{f\})$, and this set incurs partial loss, since $(0,0)=(-2,2)+2f$. As a consequence, any $\kappa_2$-decisive superset of $\desirs$ must include $-f$ (and there is at least one such superset, $\{f: f\geq(-\lambda,\lambda) \text{ for some } \lambda>0\})$. Therefore, the intersection of all such supersets also includes $-f$, meaning that this intersection does not coincide with $\desirs$. 
\end{example}

With respect to $\kappa_4$ there is only one $\kappa_4$-maximal set: $\gambles\setminus\gambles^-_0$. As a consequence, the family of $\kappa_4$-coherent sets is not a maximal strong belief structure either. To see that it is not a decisive strong belief structure it suffices to note that $\gambles\setminus\gambles^-_0$ is $\kappa_4$-coherent but has no decisive superset. 
Similar comments apply to the closure operator $\kappa_3$. 

We conclude this section by showing that there are closure operators $\kappa$ for which no decisive coherent sets exist; simply take $$\kappa(\desirs)\coloneqq\begin{cases} \gambles^+ &\text{ if } \desirs \subseteq\gambles^+ \\ \gambles &\text{ otherwise}.\end{cases}$$
To see an example where no maximal set exists, we need to devise an example where the conditions of Zorn lemma are not satisfied. One such example is the following.

\begin{example}
Let $\pspace$ be a binary space. For any pair of vectors $\vec{x}$ and $\vec{y}$, let $\widehat{\vec{x},\vec{y}}$ denote the angle they determine, given by 
\[
\widehat{\vec{x},\vec{y}}\coloneqq\arccos{\frac{\vec{x}\cdot\vec{y}}{|\vec{x}|\cdot|\vec{y}|}}.
\]
Let us denote 
\[
\mathcal{G}\coloneqq\left\{\desirs\subseteq(\gambles^-_0)^c: (\forall{\vec{x},\vec{y}\in\desirs})\ \frac{\vec{x}}{|\vec{x}|}+\frac{\vec{y}}{|\vec{y}|}\notin\gambles^-_0, \sup_{\vec{x},\vec{y}\in\desirs} \widehat{\vec{x},\vec{y}}<180^\circ\right\},
\]
and let us define the closure operator $\kappa$ by:
 \[
 \kappa(\desirs)\coloneqq\begin{cases}
  \kappa_1(\desirs \cup \gambles^+) &\text{ if } \desirs\in\mathcal{G}\\ 
  \gambles &\text{otherwise}.
  \end{cases}
 \]
It can be checked that $\kappa$ satisfies axioms~\ref{C1}--\ref{C4}: 
\begin{itemize}
    \item[\ref{C1}.] $\desirs\subseteq\kappa(\desirs)$ by construction.  
    \item[\ref{C2}.] This follows from the fact that if $\desirs\subseteq\desirs'$ and $\desirs'\in\mathcal{G}$, then also $\desirs\in\mathcal{G}$.
    \item[\ref{C3}.] This is a consequence of the implication $\desirs\in\mathcal{G} \Rightarrow \desirs\cup\gambles^+\in\mathcal{G}$.
    \item[\ref{C4}.] It suffices to apply that $\kappa_1$ satisfies~\ref{C4} and $\gambles$ is trivially closed under dominance.
\end{itemize}
To prove that $\overline{\Lambda}_{\kappa}$ is empty, consider a $\kappa$-coherent set $\desirs=\kappa(\desirs)$. Then by construction $\sup_{\vec{x},\vec{y}\in\desirs} \widehat{\vec{x},\vec{y}}<180^\circ$. Take $\vec{x'},\vec{y'}\in\overline{\desirs}$ such that $\widehat{\vec{x'},\vec{y'}}=\sup_{\vec{x},\vec{y}\in\desirs} \widehat{\vec{x},\vec{y}}$, where the closure is taken in the topology of pointwise convergence. If we denote $\vec{x'_\eps}=\vec{x'}-\eps$ and $\vec{y'_\eps}=\vec{y'}-\eps$ for $\eps>0$, it follows that there is some $\eps>0$ such that $\widehat{\vec{x'_\eps},\vec{y'_\eps}}<180^\circ$, and as a consequence $\desirs$ is strictly included in the set $\kappa_1(\desirs\cup \{\vec{x'_\eps},\vec{y'_\eps}\})$,
that is $\kappa$-coherent by construction. 
\end{example}

\subsection{Conditioning and marginalisation}
Two important operations that can be performed on a $\kappa$-coherent set of gambles are those of conditioning and marginalisation. 
\begin{definition}[{\bf Marginalisation}]\label{def:margR}
Consider $\desirs\in\Lambda_{\kappa}$, and let $\partit$ a partition of $\pspace$; the $\partit$-marginal of $\desirs$ is given by its intersection with the family $\gambles_\partit$  of $\partit$-measurable gambles:
\[
\desirs_{\partit}\coloneqq \desirs\cap \gambles_\partit.
\]
\end{definition}
\noindent It is also immediate that the marginal set of a $\kappa$-coherent set of gambles $\desirs$ is $\kappa$-coherent relative to $\gambles_\partit$, since 
\begin{equation*}
 \gambles_\partit\cap \kappa (\desirs_\partit \cup \gambles^+)\subseteq \gambles_\partit\cap \kappa (\desirs \cup \gambles^+)=\gambles_\partit \cap \desirs=\desirs_\partit,
\end{equation*}
where the inclusion follows by~\ref{C2} and the one but last equality follows from the coherence of $\desirs$.

Conditioning may be used to represent gambles that are called off if an event $B$ turns out to be false. This implies that the desirability of a gamble $f$ should only depend on its values on $B$, and therefore we may assume without loss of generality that the gamble is equal to $0$ outside $B$: our wealth shall not change in those (non-admissible) cases. This leads to the following notion:  

\begin{definition}[{\bf Conditioning}]\label{def:condR}
Consider $\desirs\in\Lambda_{\kappa}$ and let $B$ be a nonempty subset of $\pspace$. Let $\gambles|B\coloneqq\{Bf:f\in\gambles\}$ be the gambles that equal zero outside $B$. The set
$\desirs$ conditional on $B$ is defined as 
\begin{equation*}\label{eq:condit-gambles}
\desirs|B\coloneqq\desirs\cap\gambles|B=\{f\in\desirs:f=Bf\}.
\end{equation*}
\end{definition}
\noindent The conditional set derived from a $\kappa$-coherent set $\desirs$ is $\kappa$-coherent relative to $\gambles|B$: it suffices to observe that
\begin{equation*}
 \gambles|B\cap \kappa (\desirs|B \cup \gambles^+)\subseteq \gambles|B\cap \kappa (\desirs \cup \gambles^+)=\gambles|B \cap \desirs=\desirs|B,
\end{equation*}
where the inclusion follows by~\ref{C2} and the one but last equality from the coherence of $\desirs$.

Conditional sets of gambles are most often used in combination with a partition $\partit$ of $\pspace$. In that case, the conditional information along the elements of the partition can be aggregated in a single set of gambles on $\pspace$ as follows:
\begin{equation}
\desirs|\partit\coloneqq\left\{f\in\gambles(\pspace):(\forall B)\ Bf\in\desirs|B\cup\{0\}\right\}\setminus\{0\}.\label{eq:fdesirs}
\end{equation}
However this set of gambles is not automatically $\kappa$-coherent as the following example shows.
\begin{example}\label{ex:congnatex-does-not-exist} Take $\pspace\coloneqq\{\omega_1,\omega_2,\omega_3,\omega_4\}$ and the closure operator given by:
\begin{equation*}
\kappa(\desirs)\coloneqq\begin{cases}
\kappa_4(\desirs)\ &\text{ if }(\forall f \in \desirs)|\{\omega\in\pspace: f(\omega)<0\}|\leq1\\
\gambles &\text{ otherwise}.
\end{cases}
\end{equation*}
It is not difficult to prove that $\kappa$ satisfies axioms~\ref{C1}--\ref{C4}, i.e., that it belongs to $\mathcal{K}_d$. 

Now, let $\partit\coloneqq\{B,B^c\}$, with $B\coloneqq\{\omega_1,\omega_2\}$, and 
\begin{align*}
\desirs|B&\coloneqq\{f\in\gambles: f(\omega_2)>0, f(\omega_3)=f(\omega_4)=0\},\\
\desirs|B^c&\coloneqq\{f\in\gambles: f(\omega_4)>0, f(\omega_1)=f(\omega_2)=0\},\\ \desirs&\coloneqq\desirs|B\cup\desirs|B^c\cup \gambles^+.
\end{align*}
We get that $\kappa(\desirs\cup\gambles^+)=\desirs$ and $\desirs\cap\gambles^-_0=\emptyset$, so this set is $\kappa$-coherent. However, by Eq.~\eqref{eq:fdesirs}, $\desirs|\partit$ contains the gamble $(-1,1,-1,1)$, so $\kappa(\desirs|\partit)=\gambles$ and $ \desirs|\partit$ incurs partial loss. 
\end{example}
The example above motivates the question of whether or not we should consider rational a set, such as $\desirs$, which leads to a conditional set $\desirs|\partit$ that incurs partial loss. We have argued at length in some previous work (see~\cite[Sect.~6.4]{zaffalon2013a}, \cite[Sect.~6]{zaffalon2018a}, and in particular \cite[Sect.~8]{zaffalon2021}) that this should not be considered rational as long as, loosely speaking, one uses conditioning to automatically compute future beliefs and values from $\desirs$. In that case, we have rather argued in favor of imposing the following condition:
\begin{definition}[{\bf Conglomerability}]\label{def:cong}
A $\kappa$-coherent set of desirable gambles $\desirs\subseteq\gambles(\pspace)$ is said to be \emph{conglomerable} with respect to a partition $\partit$ of $\pspace$ if and only if
\begin{equation*}
\desirs|\partit\subseteq\desirs.
\end{equation*}
\end{definition}
\noindent Note that conglomerability, with $\desirs$ being coherent, implies that $\desirs|\partit$ avoids partial loss. 

The question of conglomerability was raised long ago by de Finetti \cite{finetti1931} for traditional probability ($\kappa_1$) in the case of infinite partitions $\partit$, given that conglomerability automatically holds in the finite case due to additivity (\ref{D3}).

Ex.~\ref{ex:congnatex-does-not-exist} gives a new twist to the question, as it shows that conglomerability is an issue even for finite partitions in the case of nonlinear closure operators. We regard this as further evidence that conglomerability should be imposed as an additional axiom on desirability whenever we understand `updating' as the automatic computation of future sets of desirable gambles.

\subsection{Marginal extension}
Marginal and conditional information are often assessed separately and only later they are combined into an overall set of desirable gambles $\desirs$. That is, we assess a marginal set of $\partit$-measurable gambles $\desirs_{\partit}$ that is $\kappa$-coherent with respect to $\gambles_\partit$, and for each $B\in\partit$ we assess a $\kappa$-coherent conditional set of gambles $\desirs|B$ relative to $\gambles|B$; using Eq.~\eqref{eq:fdesirs}, we gather these in the set $\desirs|\partit$. The natural extension of $\desirs_{\partit}$ and $\desirs|\partit$ is called their `marginal extension' and it extends all those assessments jointly:
\begin{definition}[{\bf Marginal extension}]
The \emph{marginal extension} of $\desirs_{\partit},\desirs|\partit$ is given by $\edesirs_{\kappa}(\desirs_{\partit}\cup\desirs|\partit)$. 
\end{definition}

The procedure of marginal extension generalises the law of total probability, and can serve as a basis for the modelling of imprecise stochastic processes \cite{cooman2008b,miranda2007}. For instance, when we consider two variables $X_1,X_2$ that are observed in succession, we initially provide our assessment about $X_1$ and then about $X_2$ given the value of $X_1$. In such a case, we would be in the situation above, with the partition $\partit$ being the set of possible values for $X_1$. 

It can be checked that the marginal extension is not coherent in general (see for instance Ex.~\ref{ex:congnatex-does-not-exist}); our next proposition gives a sufficient condition for its coherence: 
\begin{proposition}
Consider a partition $\partit$ of $\pspace$ and let $\kappa\in\mathcal{K}_d$ be a closure operator satisfying that $\kappa(\desirs)\subseteq\kappa_1(\desirs)$ for any set of gambles $\desirs$. Let $\desirs_{\partit},\desirs|B$ ($B\in\partit$) be $\kappa$-coherent sets of gambles relative to $\gambles_{\partit}$ and $\gambles|B$, respectively, and let $\desirs|\partit$ be defined by Eq.~\eqref{eq:fdesirs}. Let $\desirs':=\edesirs_{\kappa}(\desirs_{\partit}\cup\desirs|\partit)$ be their marginal extension. Then:
\begin{itemize}
    \item[(a)] $\desirs'$ is $\kappa$-coherent.
    \item[(b)] The marginal and conditional sets of gambles associated with $\desirs'$ are $\desirs_{\partit}$ and $\desirs|\partit$, respectively.
\end{itemize}
\end{proposition}

\begin{proof}
Let us first establish the result for the closure operator $\kappa_1$. 
\begin{itemize}
    \item[(a)] In the case of $\kappa_1$, $\desirs'$ is given by 
    \[
    \desirs'=\{f\geq g+h: g\in\desirs_{\partit}\cup\{0\},h\in\desirs|\partit\cup\{0\}\}\setminus\{0\}.
    \]
    Assume ex-absurdo that there is some gamble $f$ in $\desirs'\cap\gambles^-_0$. Then there are gambles $g\in\desirs_{\partit}\cup\{0\},h\in\desirs|\partit\cup\{0\}$ such that $f\geq g+h$. If $g=0$, then it must be $h\neq 0$ and since $\desirs|{\partit}$ is closed under dominance then $f\in\desirs|{\partit}$; but then for any $B\in\partit$ such that $Bh\neq 0$ it should be $Bh\in\gambles^-_0$, contradicting the coherence of $\desirs|B$. On the other hand, if $g\neq 0$, then there must be some $B\in\partit$ such that $g(B)>0$; but then $h$ restricted to $B$ must be less than zero, a contradiction with the coherence of $\desirs|B$. 
    \item[(b)] Consider a gamble $f\in\desirs'$, and let $g\in\desirs_{\partit}\cup\{0\},h\in\desirs|\partit\cup\{0\}$ be such that $f\geq g+h$. If $f\in\desirs_{\partit}$, then the gamble $h'$ given by 
        \[(\forall\omega \in B)(\forall B\in\partit)\ h'(\omega):=\sup_{\omega'\in B} h(\omega') \]
    satisfies that $h'\geq h$, whence $h'\in\desirs|\partit$, and also $f\geq g+h'$. But the coherence of $\desirs|B$ for every $B\in\partit$ implies that $h'$ must be  non-negative, whence $f\geq g\in\desirs_{\partit}$ and as a consequence $f\in\desirs_{\partit}$. 
    
    On the other hand, if $f=Bf$, then $f\geq g(B)+Bh$. It must be $g(B)\geq 0$, since otherwise there would be some $B'\neq B$ such that $g(B')>0$, whence $h$ restricted to $B'$ would be less than $0$, a contradiction. From this we deduce that $f\geq Bh$ and as a consequence $f\in\desirs|B$. 
\end{itemize}
If we now consider another operator $\kappa$ such that $\kappa(\desirs)\subseteq\kappa_1(\desirs)$ for all $\desirs$, it holds that
\[
\edesirs_{\kappa}(\desirs_{\partit}\cup\desirs|\partit)\cap\gambles_0^-\subseteq\edesirs_{\kappa_1}(\desirs_{\partit}\cup\desirs|\partit)\cap\gambles_0^-=\emptyset,
\]
whence $\edesirs_{\kappa}(\desirs_{\partit}\cup\desirs|\partit)$ is coherent. Moreover, 
\[
\desirs_{\partit}\subseteq\edesirs_{\kappa}(\desirs_{\partit}\cup\desirs|\partit)\cap\gambles_{\partit}\subseteq\edesirs_{\kappa_1}(\desirs_{\partit}\cup\desirs|\partit)\cap\gambles_{\partit}=\desirs_{\partit}
\]
and 
\[
\desirs|B\subseteq\edesirs_{\kappa}(\desirs_{\partit}\cup\desirs|\partit)\cap\gambles|B\subseteq\edesirs_{\kappa_1}(\desirs_{\partit}\cup\desirs|\partit)\cap\gambles|B=\desirs|B
\]
for any $B\in\partit$, whence the marginal and conditional sets of gambles it induces are $\desirs_{\partit}$ and $\desirs|\partit$, respectively. 
\end{proof}

In particular, the result applies to the closure operators $\kappa_2,\kappa_3,\kappa_4$. 

\section{Lower previsions}\label{sec:lpr}

A \emph{lower prevision} is a real-valued functional $\lpr$ defined on a set of gambles. In Williams-Walley's theory, it can be given a behavioural interpretation, so that $\lpr(f)$ is understood as the supremum acceptable buying price for the gamble $f$. Since buying $f$ for the price $\mu$ means that our wealth changes by $f-\mu$, the interpretation is that this should be an acceptable transaction for our subject. This leads to the following definition:

\begin{definition}[{\bf Lower prevision induced by $\desirs$}]
Given $\kappa\in\mathcal{K}_d$ and $\desirs\in\Lambda_{\kappa}$, the associated lower prevision $\lpr_{\desirs}:\gambles\rightarrow\mathbb{R}$ is given by: 
\begin{equation}\label{eq:lpr-from-desirs}
 \lpr_{\desirs}(f)\coloneqq\sup\{\mu : f-\mu\in\desirs\}.
\end{equation}
\end{definition}

In general, sets of desirable gambles are more informative than lower previsions, in that the correspondence between $\kappa$-coherent sets and lower previsions is many-to-one: two different $\kappa$-coherent $\desirs_1\neq\desirs_2$ can induce the same lower prevision via Eq.~\eqref{eq:lpr-from-desirs}. This holds for instance with the linear utility scale ($\kappa_1$), and this extra layer of information makes sets of desirable gambles useful in order to deal with the problem of sets of lower probability zero \cite{couso2011,debock2015b,cooman2012b,miranda2010c}.

Given a $\kappa_1$-coherent set of desirable gambles $\desirs$, the lower prevision it induces satisfies the following axioms for any $f,g\in\gambles$ and any $\lambda>0$:
\begin{enumerate}[label=\upshape P\arabic*.,ref=\upshape P\arabic*]
\item\label{COH1} $\lpr_{\desirs}(f)\geq \inf f$;
\item\label{COH2} $\lpr_{\desirs}(f+g)\geq \lpr_{\desirs}(f)+\lpr_{\desirs}(g)$;
\item\label{COH3} $\lpr_{\desirs}(\lambda f)=\lambda \lpr_{\desirs}(f)$.
\end{enumerate}
In fact, in Walley's theory a lower prevision satisfying axioms~\ref{COH1}--\ref{COH3} is simply called coherent. 

Let us consider an arbitrary closure operator $\kappa$ and a $\kappa$-coherent set of gambles $\desirs$, and let $\lpr_{\desirs}$ be the lower prevision it induces by means of Eq.~\eqref{eq:lpr-from-desirs}. Since by~\ref{K1} and~\ref{K2} any $\kappa$-coherent set of gambles includes $\gambles^+$ and excludes $\gambles^-_0$, axiom \ref{COH1} is always satisfied, and we also have $\lpr_{\desirs}(f)\leq\sup f$ for any gamble $f$. As a consequence, it is $\lpr_{\desirs}(\mu)=\mu$ for any real number $\mu$. In addition, the lower prevision $\lpr_{\desirs}$ also satisfies the following property:
\begin{definition}[{\bf Constant additivity}]
Functional $\lpr:\gambles\rightarrow\reals$ is said to be \emph{constant additive} if and only if 
\begin{equation}\label{eq:constant-additivity}
(\forall f \in\gambles, \mu\in\reals)\ \lpr(f+\mu)=\lpr(f)+\mu.
\end{equation}
\end{definition}

It is easy to prove that axioms~\ref{COH2},~\ref{COH3} do not hold in general: 
\begin{example}\label{ex:lpr-from-desirs}
Let $\pspace$ be a binary space. Consider the gambles $f\coloneqq(-1,1),g\coloneqq(1,-2)$ and the closure operator $\kappa_3$. Consider the set 
\[
\desirs\coloneqq\edesirs_{\kappa_3}(\{f,g\})=\gambles^+\cup \{h:(\exists \lambda>0)\ h\geq \lambda f \text{ or } h\geq \lambda g\};
\]
if we consider $h_1\coloneqq(-2,3)$ and $h_2\coloneqq(3,-2)$, we obtain  
\begin{align*}
&\sup\{\mu: h_1-\mu \in \edesirs_{\kappa_3}(\{f,g\})\}=1/2 \\
&\sup\{\mu: h_2-\mu \in \edesirs_{\kappa_3}(\{f,g\})\}=4/3\\
&\sup\{\mu: h_1+h_2-\mu \in \edesirs_{\kappa_3}(\{f,g\})\}=1,
\end{align*}
meaning that the lower prevision $\lpr_{\desirs}$ defined from $\desirs$ by Eq.~\eqref{eq:lpr-from-desirs} does not satisfy \ref{COH2}. 

On the other hand, if we apply Eq.~\eqref{eq:lpr-from-desirs} on $\edesirs_{\kappa_4}(\{f\})$ we obtain 
\begin{align*}
&\sup\{\mu: f-\mu \in \edesirs_{\kappa_4}(\{f\})\}=0 \\
&\sup\{\mu: 2f-\mu \in \edesirs_{\kappa_4}(\{f\})\}=-1,
\end{align*}
meaning that \ref{COH3} does not hold. 
\end{example}

This is not surprising, since conditions~\ref{COH1}--\ref{COH3} are a consequence of the coherence conditions encompassed by the closure operator $\kappa_1$: \ref{COH1} means that a gamble that does not make us lose utiles should always be desirable; \ref{COH2} means that if two transactions are desirable, we should be disposed to desire the combined transaction; and \ref{COH3} implies that the set of desirable gambles should be a cone, so $f$ is desirable if and only if $\lambda f$ is desirable, for any positive $\lambda$. These observations lead us to the following conclusions (see \cite[Thm.~6]{miranda2010c}): 
\begin{itemize}
\item If any coherent set is closed under finite additions (so in particular for $\kappa_2$-coherent sets), then the lower prevision it induces satisfies \ref{COH2}. 
 \item If any coherent set is closed under positive homogeneity (so in particular for $\kappa_3$-coherent sets), then \ref{COH3} holds.
  \item Since we are focusing on closure operators $\kappa\in\mathcal{K}_d$, it follows from \ref{C4} that $\lpr$ is monotone. 
\end{itemize}

Monotonicity of $\lpr$ appears to be a minimal requirement if we want to interpret the lower prevision of a gamble $f$ as its supremum acceptable buying price; indeed, if a gamble $f$ dominates another gamble $g$ it seems hard to escape that we should be disposed to pay at least as much for $f$ as for $g$, given also that the difference $f-g$ belongs to the set of desirable gambles $\gambles^+$. 

The \emph{upper} prevision associated with $\desirs$ is given by 
\begin{equation}\label{eq:upper-prev}
 \upr_{\desirs}(f):=\inf\{\mu: \mu-f\in\desirs\},
\end{equation}
and it can be interpreted as the infimum acceptable \emph{selling} price for $f$. The upper and lower previsions associated with the same set of gambles are conjugate: it holds that $\upr_{\desirs}(f)=-\lpr_{\desirs}(-f)$ for any gamble $f$. Moreover, since any coherent set of gambles $\desirs$ is closed under dominance, for any gamble $f$ it holds that:
\begin{itemize}
\item The set $B_f\coloneqq\{\mu: f-\mu\in\desirs\}$ of acceptable buying prices is a lower set: $\mu\in B_f$ implies that $\mu'\in B_f$ for every $\mu'<\mu$.
\item The set $S_f\coloneqq\{\mu: \mu-f\in\desirs\}$ of acceptable selling prices for $f$ is an upper set: $\mu\in S_f$ implies that $\mu'\in S_f$ for every $\mu'>\mu$.
\end{itemize}

If we want to interpret $\lpr_{\desirs}(f)$ and $\upr_{\desirs}(f)$ as the supremum acceptable buying price and infimum acceptable selling price for $f$, it makes sense that  $\lpr_{\desirs}(f)\leq\upr_{\desirs}(f)$: indeed, if it was $\upr_{\desirs}(f)<\lpr_{\desirs}(f)$ then for any $0<\eps<\lpr_{\desirs}(f)-\upr_{\desirs}(f)$ we should be disposed to buy the gamble $f$ for the price $\lpr_{\desirs}(f)-\frac{\eps}{2}$ and sell it for $\upr_{\desirs}(f)+\frac{\eps}{2}$, but the combination of these two transactions makes us subject to a sure loss. 

The inequality between the lower and the upper previsions induced by a set of gambles is characterised in the following proposition: 

\begin{proposition}\label{prop:lpr-leq-upr}
Consider $\kappa\in\mathcal{K}_d$, $\desirs\in\Lambda_{\kappa}$ and let $\lpr_{\desirs},\upr_{\desirs}$ be the lower and upper previsions it induces by means of Eqs.~\eqref{eq:lpr-from-desirs} and \eqref{eq:upper-prev}. Then \begin{equation}\label{eq:lpr-leq-upr}
(\forall f\in\gambles )\ \lpr_{\desirs}(f)\leq\upr_{\desirs}(f)\Leftrightarrow (\nexists g_1,g_2\in\desirs,\eps>0)\ g_1+g_2=-\eps.
\end{equation}
\end{proposition}

\begin{proof}
For the direct implication, assume the existence of such $g_1,g_2$. Then $g_2\in\desirs$ implies that $\lpr_{\desirs}(g_2)\geq 0$; while $g_1=-\eps-g_2\in\desirs$ implies that $\upr_{\desirs}(g_2)\leq -\eps<0\leq\lpr_{\desirs}(g_2)$, a contradiction. 

To prove the converse, consider a gamble $f$ such that $\lpr_{\desirs}(f)>\upr_{\desirs}(f)$. Then given $\eps:=\frac{\lpr_{\desirs}(f)-\upr_{\desirs}(f)}{4}>0$, it holds that $f-\lpr_{\desirs}(f)+\eps\in\desirs$, $\upr_{\desirs}(f)+\eps-f\in\desirs$ and their sum is $-2\eps$, a contradiction.
\end{proof}

In the remainder of this section, we shall consider only $\kappa$-coherent sets of desirable gambles satisfying Eq.~\eqref{eq:lpr-leq-upr}. These include in particular those sets $\desirs$ such that 
\begin{equation}\label{eq:both}
0\neq f\in\desirs\Rightarrow -f\notin\desirs; 
\end{equation}
to prove that Eq.~\eqref{eq:both} implies the condition in Eq.~\eqref{eq:lpr-leq-upr}, observe that if there were $g_1,g_2\in\desirs$ such that $g_1+g_2=-\eps$ for some $\eps>0$, then $-g_1=g_2+\eps\geq g_2\in\desirs$, using that any coherent set of gambles is closed under dominance, and this contradicts Eq.~\eqref{eq:both}. In particular, any $\kappa$-decisive set satisfies Eq.~\eqref{eq:both}. 

\subsection{Precise previsions}

In the traditional case of desirability, under $\posi$, lower previsions are called precise when the lower and upper previsions coincide. Precise previsions are essentially expectations, and they are linear, as it is well known. For this reason, `precise prevision' and `linear prevision' are typically used interchangeably. We shall see that linearity and precision get decoupled in the case of nonlinear desirability: precise previsions are generally nonlinear. We shall maintain the terminology `precise prevision' however, because the fundamental property of precision is actually preserved, which is the absence of indecision (which we have called decisiveness with regard to set of desirable gambles).

Our next result characterises when lower previsions are precise: 

\begin{proposition}\label{prop:lpr-upr}
Let $\desirs$ be a $\kappa$-coherent set of gambles satisfying condition~\eqref{eq:lpr-leq-upr}, and let $\lpr_{\desirs},\upr_{\desirs}$ be the lower and upper previsions it induces by means of Eqs.~\eqref{eq:lpr-from-desirs},~\eqref{eq:upper-prev}. Then \[(\forall f\in\gambles)\ \lpr_{\desirs}(f)=\upr_{\desirs}(f) \Leftrightarrow \left[(\forall f\in\gambles)\  f\notin\desirs\Rightarrow(\forall \eps>0)\ \eps-f\in\desirs\right].\]
\end{proposition}

\begin{proof}
We begin with the direct implication. Consider a gamble $f$ such that $f\notin\desirs$. Then $\lpr_{\desirs}(f)\leq 0$; since $\lpr_{\desirs}(f)=\upr_{\desirs}(f)$, we deduce from~\eqref{eq:upper-prev} that for every $\eps>0$ it must be $\eps-f\in\desirs$. 

With respect to the converse implication, note that if $\lpr_{\desirs}(f)<\upr_{\desirs}(f)$, then given $0<\eps<\upr_{\desirs}(f)-\lpr_{\desirs}(f)$, it holds that $f-\lpr_{\desirs}(f)-\frac{\eps}{2}\notin\desirs$ and $\upr_{\desirs}(f)-\frac{\eps}{2}-f\notin\desirs$; the latter gamble dominates $\frac{\eps}{2}+\lpr_{\desirs}(f)-f$, which as a consequence does not belong to $\desirs$ either. But then we have found a gamble $g:=f-\lpr_{\desirs}(f)-\frac{\eps}{4}$ such that neither $g$ and nor $\delta-g$ belongs to $\desirs$ for $0<\delta<\frac{\eps}{4}$. This a contradiction. 
\end{proof}

Let us show that precise previsions are nonlinear in general: 

\begin{example}\label{ex:precise-non-linear-binary}
Consider $\pspace\coloneqq\{\omega_1,\omega_2\}$ and the set of gambles  
$\desirs\coloneqq\gambles^+\cup\{f: f(\omega_1)<0,f(\omega_2)>1\}\cup\{f:f(\omega_1)>0,f(\omega_2)>-1\}$. Fig.~\ref{fig:linNonLin} displays the set.

\begin{figure}[h!]
\begin{tikzpicture}
\small
\draw[thick,->] (-2,0) -- (2,0) node[right] {$\omega_1$};
\draw[thick,->] (0,-1.5) -- (0,2) node[above] {$\omega_2$};
\draw[very thick,gray!50!white] (0,-1) -- (1.9,-1);
\draw[very thick,gray!50!white] (-2,1) -- (0,1);
\draw (0,1) node[right] {$1$};
\draw (0,-1) node[left] {-$1$};
\draw[draw=none,fill=gray,fill opacity=0.5] (-2,1) -- (-2,1.9) -- (1.9,1.9) -- (1.9,-1) -- (0,-1) -- (0,1) --  cycle;
\end{tikzpicture}
\caption{Set of desirable gambles $\desirs$.}\label{fig:linNonLin}
\end{figure}

This set is $\kappa_4$-coherent, because any of the three sets that build it is closed under dominance. Let $\lpr_{\desirs},\upr_{\desirs}$ be the lower and upper previsions it induces, and let us prove that $\lpr_{\desirs}=\upr_{\desirs}$. 

On the one hand, by construction it is not possible to find a gamble $f$ such that $f,-f\in\desirs$; indeed, if $f\in\desirs$, then one of the following cases must hold: 
\begin{itemize}
    \item If $f\in\gambles^+$, then $-f\in\gambles_0^-$ and as a consequence it does not belong to $\desirs$.
    \item If $f(\omega_1)<0$ and $f(\omega_2)>1$, then $-f(\omega_1)>0$ and $-f(\omega_2)<-1$, meaning that $-f\notin\desirs$.
    \item Finally, if $f(\omega_1)>0$ and $f(\omega_2)>-1$, a similar reasoning shows that $-f\notin\desirs$.
\end{itemize}
We see then that $\desirs$ satisfies~\eqref{eq:both}, whence by Prop.~\ref{prop:lpr-leq-upr} $\lpr_{\desirs}\leq\upr_{\desirs}$.
    
Next, if $f\notin\desirs$, then either: 
\begin{itemize}
    \item $f\in\gambles^-_0$, whence $-f$ belongs to $\gambles^+\cup\{0\}$ and therefore $\eps-f$ belongs to $\desirs$ for all positive $\eps$;
    \item $f(\omega_1)<0<f(\omega_2)$ and $f(\omega_2)\leq 1$, whence $-f(\omega_2)\geq -1$ and therefore $\eps-f(\omega_2)>-1$, implying that $\eps-f\in\desirs$ for all $\eps>0$.
    \item If $f(\omega_1)>0>f(\omega_2)$ and $f(\omega_2)\leq -1$, with a similar reasoning we conclude that $\eps-f\in\desirs$ for all $\eps>0$. 
\end{itemize}
Applying Prop.~\ref{prop:lpr-upr}, we deduce that $\lpr_{\desirs}=\upr_{\desirs}$.

However, if we denote by $\pr\coloneqq\lpr_{\desirs}=\upr_{\desirs}$, this prevision is not linear: we have for instance that $\pr(-2,1)=0$ while $\pr(-1,0.5)=-0.5$. 
\end{example}

Remember that in the case of linear utility scale ($\kappa_1$), maximality and decisiveness of a set of desirable gambles are equivalent notions. Because of this reason, a $\kappa_1$-maximal set of gambles always induces a precise prevision, and conversely for any precise prevision it is always possible to determine a $\kappa_1$-maximal set of gambles that induces it. 

In the case of arbitrary closure operators, the equivalence of maximality and decisiveness breaks down, and only the latter notion relates to precision. Consequently, given a maximal set of gambles $\desirs$ and the lower and upper previsions $\lpr,\upr$ it induces by means of Eqs.~\eqref{eq:lpr-from-desirs} and~\eqref{eq:upper-prev}, it need not hold that $\lpr=\upr$. To prove this, it suffices to notice that for $\kappa=\kappa_4$ or $\kappa=\kappa_3$ the only maximal set is $\desirs=\gambles\setminus\gambles^-_0$, which induces $\lpr(f)=\sup f$ and $\upr(f)=\inf f$ for any gamble $f$. On the other hand, it follows from Prop.~\ref{prop:lpr-upr} that for any $\kappa$-decisive set the lower and upper previsions it induces coincide.  

\subsection{Sets of desirable gambles associated with a lower prevision}

As we said, sets of desirable gambles constitute a more informative model than lower previsions, in the sense that two different coherent sets $\desirs_1,\desirs_2$ may induce the same $\lpr$ by means of Eq.~\eqref{eq:lpr-from-desirs}. It can be checked that the sets $\desirs$ that induce a given $\lpr$ by means of \eqref{eq:lpr-from-desirs} are the $\kappa$-coherent sets satisfying 
\[
\gambles^+\cup\{f: \lpr(f)>0\}\subseteq\desirs\subseteq\{f:\lpr(f)\geq 0\}.
\]
The upper bound in the above chain of inclusions is not a $\kappa$-coherent set, because it incurs partial loss, due to the equality $\lpr(0)=0$. It can be checked nonetheless that under some conditions it corresponds to the topological closure of the set $\desirs$: 

\begin{proposition}
Let $\pspace$ be a finite set, $\kappa\in\mathcal{K}_d, \desirs\in\Lambda_{\kappa}$ and $\lpr_{\desirs}$ the lower prevision it induces by means of Eq.~\eqref{eq:lpr-from-desirs}. Moreover, let $\overline{\desirs}$ denote the closure of $\desirs$ in the topology of pointwise convergence. Then $\overline{\desirs}=\{f: \lpr_{\desirs}(f)\geq 0\}$.
\end{proposition}

\begin{proof}
We begin with the direct inclusion. By construction, $\lpr_{\desirs}(f)\geq 0$ for every $f\in\desirs$. Consider a gamble $f\in\overline{\desirs}$, and let $(f_n)_n$ be a sequence of gambles of $\desirs$ that converges pointwise to $f$. Since by assumption $\pspace$ is finite, this means that $(f_n)_n$ also converges uniformly towards $f$, whence for any $\eps>0$ there is some $n_{\eps}$ such that $\|f_n-f\|<\eps$ for every $n\geq n_{\eps}$. Thus, $f_n\leq f+\eps$ for every $n\geq n_{\eps}$, and since $\desirs$ is closed under dominance this means that $f+\eps\in\desirs$ for every $\eps>0$. Using Eq.~\eqref{eq:constant-additivity}, we deduce that $\lpr_{\desirs}(f)\geq-\eps$ for every $\eps>0$, and therefore that $\lpr_{\desirs}(f)\geq 0$. 

To prove the converse inclusion, consider a gamble $f$ such that $\lpr_{\desirs}(f)\geq 0$, and let us show that $f\in\overline{\desirs}$. For any $\eps>0$, it follows from Eq.~\eqref{eq:constant-additivity} that $\lpr_{\desirs}(f+\eps)\geq\eps>0$, whence 
$f+\eps$ belongs to $\desirs$ for every $\eps>0$ and therefore its limit $f$ as $\eps \downarrow 0$ belongs to $\overline{\desirs}$. 
\end{proof}

With respect to the lower bound, given a lower prevision $\lpr$ (not necessarily induced by some coherent set), under quite mild conditions on $\lpr$ it can be showed to be coherent for some closure operator:  

\begin{proposition}\label{prop:desirs-pr}
Let $\lpr:\gambles\rightarrow\reals$ be a monotone functional, and consider the set of gambles 
\begin{equation}\label{eq:strict-desirable}
\desirs_{\lpr}\coloneqq\gambles^+\cup\{f:\lpr(f)>0\}.
\end{equation} 
Then the following are equivalent: 
\begin{itemize}
    \item[(a)] $\desirs_{\lpr}$ is $\kappa$-coherent for some closure operator $\kappa\in\mathcal{K}_d$.
    \item[(b)] $\desirs_{\lpr}$ is closed under dominance and avoids partial loss.
    \item[(c)] $\desirs_{\lpr}$ avoids partial loss. 
\end{itemize}
As a consequence, if $\lpr(0)=0$ then $\desirs_{\lpr}$ is coherent. 
\end{proposition}

\begin{proof}
To prove that (a) and (b) are equivalent, note that since $\kappa_4$ is the smallest  closure operator in $\mathcal{K}_d$, then $\desirs_{\lpr}$ is $\kappa$-coherent for some $\kappa\in\mathcal{K}_d$ if and only if it is $\kappa_4$-coherent. 

The implication (b)$\Rightarrow$(c) is trivial. Finally, note that given $g\geq f\in\desirs_{\lpr}$, then either $f\in\gambles^+$, whence also $g\in\gambles^+\subseteq\desirs_{\lpr}$, or $\lpr(f)>0$, in which case $\lpr(g)\geq\lpr(f)>0$, whence also $g\in\desirs_{\lpr}$. Thus, $\desirs_{\lpr}$ is always closed under dominance and as a consequence (c) implies (b).

Finally, if $\lpr(0)=0$ then by monotonicity it is $\lpr(f)\leq 0$ for every $f\in\gambles^-_0$, whence $\desirs_{\lpr}$ avoids partial loss. Applying the first part of the proof we deduce that it is $\kappa_4$-coherent.
\end{proof}

\subsection{Consistency notions}

The behavioural interpretation of lower previsions as supremum acceptable buying prices gives rise to two consistency notions in the case of linear utility scale: \emph{avoiding sure loss} and \emph{coherence}. The former entitles that a combination of acceptable transactions should not make us subject to a sure loss, while the latter means that the supremum acceptable buying prices should not increase by taking into consideration other acceptable transactions. In both cases, the notion depends on how acceptable transactions can be combined, which is the information encompassed by the closure operator. This leads us to study how to define these properties for an arbitrary $\kappa$. 

Let us begin with the notion of avoiding sure loss. As we mentioned, in the case of linear utility scale, we say that a lower prevision $\lpr$ avoids sure loss if and only if it is not possible to combine a number of gambles whose desirability follows from the assessments in $\lpr$ and obtain a sure loss. This can be formulated by saying that, for every $f_1,\dots,f_n$ in $\gambles$ and every non-negative $\lambda_1,\dots,\lambda_n$, it holds that
\[
\sup \left[\sum_{i=1}^{n} \lambda_i(f_i-\lpr(f_i))\right]\geq 0.
\]
The following definition generalises this idea for lower previsions associated with arbitrary closure operators:

\begin{definition}[{\bf Avoiding sure loss for lower previsions}]
Given $\kappa\in\mathcal{K}_d$ and a lower prevision $\lpr$, we say that $\lpr$ \emph{$\kappa$-avoids sure loss} if and only if the set $\desirs_{\lpr}$ given by Eq.~\eqref{eq:strict-desirable} does.
\end{definition}

The above definition takes a simpler form in some particular cases of closure operators: 

\begin{definition}[{\bf Finitary and continuous closure operators}]
A closure operator $\kappa\in\mathcal{K}_d$ is called \emph{finitary} if and only if for any set of gambles $\desirs$ it holds that 
\begin{equation*}
 \edesirs_{\kappa}(\desirs)=\bigcup_{\substack{\desirs'\subseteq\desirs \\ \text{ finite}}} \edesirs_{\kappa}(\desirs'),
\end{equation*}
and it is called \emph{continuous} if and only if for any $f\notin\gambles^+$
\[
f\in\cup_{\eps>0} \edesirs_{\kappa}^{\eps}(\desirs)\Leftrightarrow f\in\cup_{\eps>0}\edesirs_{\kappa}(\desirs^{\eps}),
\]
where we denote $\mathcal{G}^{\eps}:=\{f+\eps:f\in\mathcal{G}\}$ for any $\mathcal{G}\subseteq\gambles$.
\end{definition}

All the closure operators in Ex.~\ref{ex:first-example} are finitary and continuous. When $\kappa$ is finitary, it is easy to show that a lower prevision $\lpr$ $\kappa$-avoids sure loss if and only if for every $f_1,\dots,f_n\in\gambles$ and every $\eps>0$, the set $\{f_1-\lpr(f_1)+\eps,\dots,f_n-\lpr(f_n)+\eps\}$ avoids sure loss. 

\begin{example}
Let us analyse the above definition in terms of the closure operators considered in Ex.~\ref{ex:first-example}:
\begin{itemize}
 \item[$\kappa_2$.] $\lpr$ $\kappa_2$-avoids sure loss if and only if $$\sup\left[\sum_{i=1}^{n} f_i-\lpr(f_i)\right]\geq 0 \text { for every } f_1,\dots,f_n\in\gambles.$$
 \item[$\kappa_3$.] $\lpr$ $\kappa_3$-avoids sure loss if and only if $\lpr(f)\leq \sup f$ for every $f\in\gambles$.
 \item[$\kappa_4$.] $\lpr$ $\kappa_4$-avoids sure loss if and only if $\lpr(f)\leq \sup f$ for every $f\in\gambles$.
\end{itemize}
We obtain then that the notions of $\kappa_3$- and $\kappa_4$-avoiding sure loss are equivalent, and so are\footnote{See \cite[Def.~2.4.1 \text{ and } Lem.~2.4.4]{walley1991}.} $\kappa_1$- and $\kappa_2$-avoiding sure loss. 
\end{example}

Similarly, coherence means that we should not be able to raise the supremum acceptable buying price for a gamble $f$ taking into account the implications of other desirable transactions. In the case of linear utility scale, this can be formulated by saying that, for every $f_0,f_1,\dots,f_n$ in $\gambles$ and every non-negative $\lambda_0,\lambda_1,\dots,\lambda_n$, it holds that 
\[
\sup \left[\sum_{i=1}^{n} \lambda_i(f_i-\lpr(f_i))-\lambda_0(f_0-\lpr(f_0))\right]\geq 0.
\]
This leads us to propose the following definition:

\begin{definition}[{\bf Coherence for lower previsions}]\label{def:coherence-lpr}
Consider $\kappa\in\mathcal{K}_d$ and let $\lpr$ a lower prevision. We say that $\lpr$ is \emph{$\kappa$-coherent} if and only it $\desirs_{\lpr}$ is. 
\end{definition}

When the closure operator $\kappa$ is finitary, Def.~\ref{def:coherence-lpr} is equivalent to 
\begin{multline*}
(\nexists f_0,f_1,\dots,f_n\in\gambles, \eps>0,\eps'\geq 0) \\ f_0-\lpr(f_0)-\eps'\in\edesirs_{\kappa}(\{f_1-\lpr(f_1)+\eps,\dots,f_n-\lpr(f_n)+\eps\}).
\end{multline*}

\begin{example}
With respect to the closure operators considered in Ex.~\ref{ex:first-example}, this definition can be reformulated in the following manner:
\begin{itemize}
\item[$\kappa_2$.] $\lpr$ is $\kappa_2$-coherent if and only if $\lpr(f)\geq \inf f$ and $\sup[\sum_{i=1}^{n} (f_i-\lpr(f_i))-(f-\lpr(f))]\geq 0$ for every $f_1,\dots,f_n,f\in\gambles$.
 \item[$\kappa_3$.] $\lpr$ is $\kappa_3$-coherent if and only if $\lpr(f)\geq \inf f$ and $\sup[\lambda(f'-\lpr(f'))-(f-\lpr(f))]\geq 0$ for every $f,f'\in\gambles$, $\lambda>0$.
 \item[$\kappa_4$.] $\lpr$ is $\kappa_4$-coherent if and only if $\lpr(f)\geq \inf f$ and $\sup[(f'-\lpr(f'))-(f-\lpr(f))]\geq 0$ for every $f,f'\in\gambles$.
 \end{itemize}
Note that in this case $\kappa_3$- and $\kappa_4$-coherence are no longer equivalent, and neither are $\kappa_1$- and $\kappa_2$-coherence.
\end{example}

Our next result gives some properties of the lower prevision induced by a set of gambles $\desirs$ that is not necessarily $\kappa$-coherent. 

\begin{proposition}
Let $\kappa\in\mathcal{K}_d$. Consider a set of gambles $\desirs$ that includes $\gambles^+$ and let $\lpr_{\desirs}$ be the lower prevision it induces using Eq.~\eqref{eq:lpr-from-desirs}. 
\begin{itemize}
\item[(a)] $\lpr_{\desirs}$ is real-valued if $\edesirs_{\kappa}(\desirs)\neq\gambles$. 
\end{itemize}
Assume that $\kappa$ is finitary and continuous and that $\lpr_{\desirs}$ is real-valued.
\begin{itemize}
\item[(b)] $\lpr_{\desirs}$ $\kappa$-avoids sure loss $\Leftrightarrow$ $\desirs$ avoids sure loss.
\item[(c)] If $\desirs$ is $\kappa$-coherent, then $\lpr_{\desirs}$ is $\kappa$-coherent.
\end{itemize}
\end{proposition}

\begin{proof}
 \begin{itemize}
 \item[(a)] First of all, since $\gambles^+\subseteq\desirs$ by assumption, it is $\lpr_{\desirs}(f)\geq\inf f$ for every gamble $f$. As a consequence, $\lpr_{\desirs}$ is real-valued if and only if $\lpr_{\desirs}(f)<+\infty$ for any gamble $f$. 
 
If there is a gamble $f$ with $\lpr_{\desirs}(f)=+\infty$, this means that for any $\mu\in\reals$ it holds that $f-\mu\in\desirs$; since for any other gamble $g$ there is some real valued $\mu'$ such that $f-\mu'<g$, this implies that for any real $\mu$ it is $g-\mu>f-(\mu+\mu')\in\edesirs_{\kappa}(\desirs)$, taking into account that this set is closed under dominance. Therefore, $\lpr_{\edesirs_{\kappa}(\desirs)}(g)=+\infty$ for any gamble $g$, which in turn means that $\edesirs_{\kappa}(\desirs)=\gambles$. As a consequence,  if $\edesirs_{\kappa}(\desirs)\neq \gambles$ then $\lpr_{\desirs}$ is  real-valued. 
 
 \item[(b)] By definition, $\lpr_{\desirs}(f)\geq 0$ for every $f\in\desirs$. It follows that $\desirs_{\lpr}=\gambles^+\cup\{f\in\desirs: (\exists\eps>0) f-\eps\in\desirs\}$. As a consequence, 
 if $\desirs$ avoids sure loss then so does $\desirs_{\lpr}$. Conversely, if $\desirs$ incurs a sure loss then there is some $f\in\gambles_{<}$ such that $f\in\edesirs_{\kappa}(\desirs)$. As a consequence, there is some $\eps>0$ such that $f+\eps\in\gambles_{<}$, and since $\kappa$ is assumed to be continuous, it follows that there is some $\eps'>0$ such that $f+\eps\in\edesirs_{\kappa}(\desirs^{\eps'})\subseteq \edesirs_{\kappa}(\desirs_{\lpr})$, taking into account for the inclusion that
 \[
 g\in\desirs \Rightarrow \lpr_{\desirs}(g)\geq 0 \Rightarrow (\forall \eps'>0)\ \lpr_{\desirs}(g+\eps')>0 \ \Rightarrow g+\eps' \in\desirs_{\lpr} 
 \]
 for any gamble $g$, using constant additivity. This means that $\desirs_{\lpr}$ incurs a sure loss, a contradiction. From this we deduce that
 \[
 \desirs \text{ avoids sure loss } \iff \desirs_{\lpr} \text{ avoids sure loss } \iff \lpr_{\desirs} \text{ avoids sure loss}.
 \]
\item[(c)] We only need to show that if $\desirs$ is $\kappa$-coherent then so is $\desirs_{\lpr}$. For this, since 
\begin{equation}\label{eq:strict-desirable-aux}
 \desirs_{\lpr}=\gambles^+\cup\{f: (\exists \eps>0)\ f-\eps\in\desirs\}=\gambles^+\cup\bigcup_{\eps>0}\desirs^{\eps},
 \end{equation}
 it follows that $\desirs_{\lpr}$ avoids partial (and as a consequence sure) loss. On the other hand, for any gamble $f$ in $\edesirs_{\kappa}(\desirs_{\lpr})$, since $\kappa$ is finitary there exist $f_1,\dots,f_n\in\desirs,\eps_1,\dots,\eps_n>0$, such that $f\in\edesirs_{\kappa}(\{f_1+\eps_1,\dots,f_n+\eps_n\})$, using the representation in Eq.~\eqref{eq:strict-desirable-aux}. Taking $\eps:=\min\{\eps_1,\dots,\eps_n\}$ and applying dominance, we deduce that $f\in\edesirs_{\kappa}(\{f_1+\eps,\dots,f_n+\eps\})\subseteq\edesirs_{\kappa}(\desirs^{\eps})\subseteq\gambles^+\cup\bigcup_{\eps'>0} \edesirs_{\kappa}^{\eps'}(\desirs)=\gambles^+\cup\bigcup_{\eps'>0}\desirs^{\eps'}\subseteq\desirs_{\lpr}$, using that $\kappa$ is continuous. Therefore, $\desirs_{\lpr}$ is $\kappa$-coherent and as a consequence so is $\lpr_{\desirs}$. \qedhere 
\end{itemize}
\end{proof}

Finally, it is straightforward to give a notion of marginalisation and conditioning for lower previsions. 

With respect to marginalisation, given a partition $\partit$ of the possibility space $\pspace$, the marginal $\lpr_{\partit}$ of $\lpr$ is simply the restriction of $\lpr$ to the class $\gambles_{\partit}$ of $\partit$-measurable gambles. When $\lpr$ is associated with a set of desirable gambles $\desirs$ by means of Eq.~\eqref{eq:lpr-from-desirs} then for any $f\in\gambles_{\partit}$ it holds that
\begin{equation*}
 \lpr(f)=\sup\{\mu : f-\mu \in\desirs\}=\sup\{\mu : f-\mu \in\desirs_{\partit}\},
\end{equation*}
since $f-\mu\in\gambles_{\partit}$ for any real number $\mu$. In other words, the diagram in Fig.~\ref{fig:marginalisation} commutes.
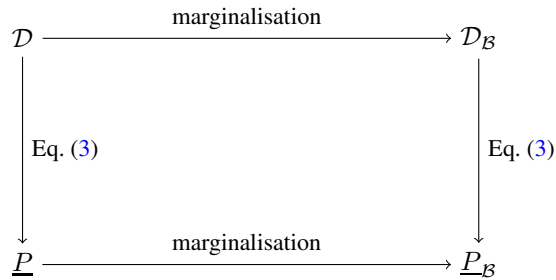
\begin{figure}[h]
\begin{center}
\begin{tikzpicture}
\draw (0,3) node(D) {$\desirs$};
\draw (6,3) node(margD) {$\desirs_{\partit}$};
\draw (0,0) node(lP) {$\lpr$};
\draw (6,0) node(marglP) {$\lpr_{\partit}$};
\draw[->] (D) --node[above] {\small marginalisation} (margD);
\draw[->] (D) --node[right] {\small Eq.~\eqref{eq:lpr-from-desirs}} (lP);
\draw[->] (lP) --node[above] {\small marginalisation} (marglP);
\draw[->] (margD) --node[right] {\small Eq.~\eqref{eq:lpr-from-desirs}} (marglP);
\end{tikzpicture}
\end{center}
\caption{Marginalisation.} \label{fig:marginalisation}
\end{figure}

Concerning conditioning with respect to an event $B$, given a lower prevision $\lpr$ and a gamble $f$ we may define $\lpr(f|B)$ as
\begin{equation*}
 \lpr(f|B)\coloneqq\begin{cases}
  \sup\{\mu : \lpr(B(f-\mu))>0\} & \text{ if } \lpr(B)>0 \\
  \inf_B f &\text{ otherwise},
  \end{cases}
\end{equation*}
which is called \emph{generalised Bayes rule} in the case case of linear utility scale. In this case there is not a one-to-one correspondence with the notion of conditioning for sets of desirable gambles: 

\begin{proposition}\label{pr:bounds-clpr}
Let $\desirs$ be a coherent set of desirable gambles and $\lpr$ the lower prevision it induces by means of Eq.~\eqref{eq:lpr-from-desirs}. If $\lpr(B)>0$ then for any gamble $f$  it holds that
\begin{equation*}
 \sup\{\mu : \lpr(B(f-\mu))>0\} \leq \sup\{\mu : B(f-\mu)\in\desirs\} \leq \sup\{\mu : \lpr(B(f-\mu))\geq 0\}.
\end{equation*}
\end{proposition}

\begin{proof}
 To prove the first inequality, note that, if $\lpr(B(f-\mu))>0$, then there is some $\eps>0$ such that $B(f-\mu)-\eps\in\desirs$, and since the latter is closed by dominance we deduce that also $B(f-\mu)\in\desirs$. Secondly, if $B(f-\mu)\in\desirs$ it follows by definition that $\lpr(B(f-\mu))\geq 0$. 
\end{proof}

However, the equality $\sup\{\mu: \lpr(B(f-\mu))>0\}=\sup\{\mu: \lpr(B(f-\mu)) \geq 0\}$ does not hold in general, although it does if $\desirs$ satisfies homogeneity and additivity (that is, in the case of linear utility scale). 

\begin{example}
Let us consider the possibility space $\pspace\coloneqq\{\omega_1,\omega_2,\omega_3\}$, the event $B\coloneqq\{\omega_1,\omega_2\}$ and the gamble $f\coloneqq(1,-1,0)$. Let us give three $\kappa_4$-coherent sets so as to illustrate that the value $\sup\{\mu: B(f-\mu)\in\desirs\}$ may agree with any of the two bounds in Prop.~\ref{pr:bounds-clpr} or with none of them: 
\begin{itemize}
    \item Consider first of all the set $\desirs_1\coloneqq\gambles^+\cup\{g: \median(g)>0\}$. If we denote by $\lpr_1$ the associated lower prevision, it holds that $\lpr_1(B)=1$ and 
    \begin{align*}
    \sup\{\mu: \lpr_1(B(f-\mu))>0\}&=-1=\sup\{\mu: B(f-\mu)\in\desirs_1\}\\&<1=\sup\{\mu: \lpr_1(B(f-\mu))\geq 0\}.
    \end{align*}
    \item Take now $\desirs_2\coloneqq\{g: \median(g)\geq 0\}\setminus\gambles^-_0$. Then its associated lower prevision $\lpr_2$ satisfies $\lpr_2(B)=1$, and moreover
    \begin{align*}
    \sup\{\mu: \lpr_2(B(f-\mu))>0\}&=-1<1=\sup\{\mu: B(f-\mu)\in\desirs_1\}\\&=\sup\{\mu: \lpr_2(B(f-\mu))\geq 0\}.
    \end{align*}
    \item Finally, let $\desirs_3\coloneqq\edesirs_{\kappa_4}(\{(0.5,-1.5,0),(0.9,-0.1,-0.1),(0.25+\delta,-0.75,\delta): \delta>0\}$. The associated lower prevision $\lpr_3$ satisfies $\lpr_3(B)=0.1>0$, and \begin{align*}
    \sup\{\mu :\lpr_3(B(f-\mu))>0\}&=0.1<0.5=\sup\{\mu: B(f-\mu)\in\desirs_1\}\\&<0.75=\sup\{\mu: \lpr_3(B(f-\mu))\geq 0\}.
    \end{align*}
\end{itemize}
\end{example}

Let us consider now the lower prevision associated with the marginal extension. Consider $\desirs_{\partit},\desirs|B$  ($B\in\partit$) that are $\kappa$-coherent with respect to $\gambles_{\partit},\gambles|B$, respectively and their associated lower previsions $\lpr_{\partit},\lpr(\cdot|\partit)$; the latter is given by 
\[
\lpr(f|\partit)\coloneqq \sum_{B\in\partit} B f_B, \text{ where }  f_B\coloneqq\lpr(f|B)=\sup\{\mu: B(f-\mu) \in\desirs|B\}.
\]
The marginal extension $\edesirs_{\kappa}(\desirs_{\partit}\cup\desirs|\partit)$ induces a lower prevision $\lpr$ that is related to $\lpr_{\partit},\lpr(\cdot|\partit)$. To see a couple of examples, note that when $\kappa=\kappa_4$,  
\begin{align*}
\edesirs_{\kappa_4}(\desirs_{\partit}\cup\desirs|\partit)&=\edesirs_{\kappa_4}(\desirs_{\partit})\cup\edesirs_{\kappa_4}(\desirs|\partit)\\&=\{f: (\exists g\in\desirs_{\partit})\ f\geq g\} \cup \{f: (\exists h\in\desirs|\partit)\ f\geq h\}.
\end{align*}
In particular, if from $\lpr_{\partit},\lpr(\cdot|\partit)$ we consider the lower previsions $\lpr_1,\lpr_2$ on $\gambles$ given by
\begin{align*}
    \lpr_1(f)&\coloneqq\sup\left\{\lpr_{\partit}(g): g\in\gambles_{\partit}, g\leq f\right\}\\
    \lpr_2(f)&\coloneqq\sup\left\{\inf_{B\in\partit}\lpr(g|B): g\in\gambles, g\leq f\right\},
\end{align*}
then the marginal extension $\edesirs_{\kappa_4}(\desirs_{\partit}\cup\desirs|\partit)$ induces the lower prevision $\lpr=\max\{\lpr_1,\lpr_2\}$.

On the other hand, when $\kappa=\kappa_1$, the marginal extension $\edesirs_{\kappa_1}(\desirs_{\partit}\cup\desirs|\partit)$ induces the coherent lower prevision $\lpr_{\partit}(\lpr(\cdot|\partit))$.

\section{Credal sets}\label{sec:credal} 

We briefly discuss next the connection between sets of desirable gambles and sets of linear previsions (what we shall call in this paper a \emph{credal set}) in the case of nonlinear utility scales. Recall that from a coherent set of desirable gambles $\desirs$ we can determine the set $B_f\coloneqq\{\mu: f-\mu\in \desirs\}$ of acceptable buying prices for a given gamble $f$, whose supremum is the value $\lpr_{\desirs}(f)$. We may alternatively summarise this information in terms of 
\[
\solp_f(\lpr_{\desirs})\coloneqq\{\pr \text{ linear prevision}: \lpr_{\desirs}(f)\leq\pr(f)\},
\]
those linear previsions whose value on $f$ dominates the lower prevision determined by $\desirs$. It follows from axioms~\ref{K1}--\ref{K2} that $\inf f\leq \lpr_{\desirs}(f)\leq \sup f$, from which $\solp_f(\lpr_{\desirs})$ is nonempty. If we summarise our information in terms of the sets $\solp_f(\lpr_{\desirs})$, we may notice that, since by construction $\lpr_{\desirs}$ satisfies constant additivity, it holds that $\solp_f(\lpr_{\desirs})=\solp_{f+\mu}(\lpr_{\desirs})$ for any real $\mu$, whence we can focus on those gambles $f$ for which $\lpr_{\desirs}(f)=0$. 

From the point of view of traditional desirability theory, the set $\solp_f(\lpr_{\desirs})$ corresponds to the natural extension of the assessment $\lpr_{\desirs}(f)$, or, equivalently, to the assessment of the supremum acceptable buying price for $f$. If we recall that any prevision $\pr$ gives a fair price $\pr(f)$ for the gamble $f$, then $\solp_f(\lpr_{\desirs})$ may be regarded as those probabilities that are compatible with the supremum acceptable buying price for $f$, in the sense that if a price $\mu$ has been deemed an acceptable buying price for $f$ by $\lpr_{\desirs}$ this should not be contradicted by the information given by $\pr$. 

In the case of linear utility scale (i.e., for $\kappa=\kappa_1$) the intersection of the credal sets $\solp_f(\lpr_{\desirs})$ for all $f\in\gambles$ is closed (under the weak-* topology) and convex, and it holds that a set of desirable gambles $\desirs$ avoids sure loss if and only if this intersection is nonempty. However, this need not be the case for arbitrary closure operators. In this respect, when the lower and upper previsions determined by $\desirs$ satisfy $\lpr_{\desirs}\leq\upr_{\desirs}$, then it also holds that $\solp_f(\lpr_{\desirs})\cap \solp_{-f}(\lpr_{\desirs})=\{\pr \text{ linear prevision}: \lpr_{\desirs}(f)\leq \pr(f)\leq \upr_{\desirs}(f)\}$ is nonempty for every $f$; the inequality $\lpr_{\desirs}\leq\upr_{\desirs}$ has been characterised in Prop.~\ref{prop:lpr-leq-upr}. It may be that $\lpr_{\desirs}\leq\upr_{\desirs}$ and that the intersection $\cap_f \solp_f(\lpr_{\desirs})=\{\pr : (\forall f\in\gambles)\ \pr(f)\geq \lpr_{\desirs}(f)\}$ is empty, though, as Ex.~\ref{ex:precise-non-linear-binary} shows. This should not be regarded as a problem, given that in the nonlinear case the intersection of the sets does not play a central role in the theory. To complement this discussion, we report some properties of this intersection for arbitrary closure operators: 
\begin{proposition}
Let $\kappa\in\mathcal{K}_d$ and let $\desirs$ be a $\kappa$-coherent set of gambles, and let $\{\solp_f(\lpr_{\desirs}): f\in\gambles\}$ be the credal sets it determines for the different gambles $f$ via the lower prevision $\lpr$ it induces. 
\begin{itemize}
\item[(a)] $\cap_f \solp_f(\lpr_{\desirs})=\{\pr:(\forall f\in\desirs)\ \pr(f)\geq 0\}=\left\{\pr: (\forall f\in\edesirs_{\kappa_1}(\desirs))\ \pr(f)\geq 0\right\}$.
\item[(b)] The set $\cap_f \solp_f(\lpr_{\desirs})$ is closed and convex. 
\item[(c)] If in addition 
\begin{equation}\label{eq:addition}
 f,g\in\desirs \Rightarrow f+g\in\desirs, 
\end{equation}
then $\cap_f \solp_f(\lpr_{\desirs})$ is nonempty. 
\end{itemize}
\end{proposition}

\begin{proof}
\begin{itemize}
\item[(a)] Let us establish the first equality; the second follows from the linearity of the elements of $\cap_f \solp_f(\lpr_{\desirs})$. 

Given $f\in\desirs$, it follows from Eq.~\eqref{eq:lpr-from-desirs} that $0\leq\lpr_{\desirs}(f)\leq\pr(f)$, whence $\cap_{f}\solp_f(\lpr_{\desirs}) \subseteq \{\pr:(\forall f\in\desirs)\ \pr(f)\geq 0\}$; conversely, given a gamble $f$ and $\eps>0$, it holds that $f-\lpr_{\desirs}(f)+\eps\in\desirs$, whence $\pr(f-\lpr_{\desirs}(f)+\eps)\geq 0$, or equivalently, $\pr(f)\geq\lpr_{\desirs}(f)-\eps$. Since this holds for any $\eps>0$, we deduce that $\pr(f)\geq\lpr_{\desirs}(f)$. Therefore, $f\in\solp_f(\lpr_{\desirs})$.

\item[(b)] Convexity holds trivially, noting that 
\[
(\alpha \pr_1+(1-\alpha)\pr_2)(f)=\alpha\pr_1(f)+(1-\alpha)\pr_2(f)\geq 0
\]
for any $\pr_1,\pr_2\in\cap_f \solp_f(\lpr_{\desirs})$, any $f\in\desirs$ and any $\alpha\in(0,1)$. 

To prove that it is also closed, note that if $\pr\notin\cap_f \solp_f(\lpr_{\desirs})$, there is some gamble $f\in\desirs$ such that $\pr(f)<0$. Then the set
\[
\left\{\pr' : |\pr'(f)-\pr(f)|<-\frac{\pr(f)}{2}\right\}
\]
is a neighbourhood in the weak-* topology that by construction is included in $(\cap_f \solp_f(\lpr_{\desirs}))^c$. This means that $(\cap_f \solp_f(\lpr_{\desirs}))^c$ is weak-* open and as a consequence $\cap_f \solp_f(\lpr_{\desirs})$ is weak-* closed. 

\item[(c)] This is a consequence of the separation lemma in \cite[Lem.~3.3.2]{walley1991}: $\cap_f \solp_f(\lpr_{\desirs})$ is nonempty if and only if for any $f_1,\dots,f_n$ in $\desirs$ it holds that $\sup \sum_{j=1}^{n} f_j\geq 0$, which in turn is equivalent to $\kappa_2(\desirs)\cap \gambles_<=\emptyset$. This is guaranteed by Eq.~\eqref{eq:addition}, considering that $\sum_{j=1}^{n} f_j\in\desirs$ and that $\desirs\cap \gambles^-_0=\emptyset$ by~\ref{K2}. \qedhere
\end{itemize}
\end{proof}

\subsection{Consistency notions}

The notions of avoiding sure loss and coherence that we have given for sets of desirable gambles or lower previsions can also be given for the families of credal sets, simply by making the appropriate transformation. In this sense, avoiding sure loss is defined in the following manner: 
\begin{definition}[{\bf Avoiding sure loss for credal sets}]
Consider a family of credal sets $\{\solp_f: f\in\gambles\}$ and let $\lpr$ be the lower prevision defined by $\lpr(f)\coloneqq \inf_{\pr\in\solp_f} \pr(f)$. Then $\{\solp_f: f\in\gambles\}$ is said to \emph{avoid sure loss} if and only if $\lpr$ does. 
\end{definition}

Equivalently, this means that the set $\{f-\mu:(\exists \pr_f\in\solp_f)\ \mu>\pr_f(f)\}$ should avoid sure loss. 

Similarly, we can give a notion of coherence: 

\begin{definition}[{\bf Coherence for credal sets}]
Consider the family of credal sets $\{\solp_f: f\in\gambles\}$ and let $\lpr$ be the lower prevision defined by $\lpr(f)\coloneqq \inf_{\pr\in\solp_f} \pr(f)$. Then $\{\solp_f: f\in\gambles\}$ is said to be \emph{coherent} if and only if $\lpr$ is. 
\end{definition}

The interpretation of this condition would again be that a combination of desirable transactions should not allow us to raise the lower prevision for another gamble $f$, meaning shrinking the credal set $\solp_f$ from its original assessment.

Equivalently, we may transform the family of credal sets into a set of desirable gambles, and analyse the notions of avoiding sure loss and coherence considering the results from Section~\ref{sec:scale}. In this respect, note that any credal set $\solp_f$ determines a set of gambles that we consider desirable, namely
\begin{equation}\label{eq:desirs-from-credal}
\desirs_f\coloneqq\{f-\lpr(f)+\eps: \eps>0\},
\end{equation}
where $\lpr(f)\coloneqq\inf_{\pr\in\solp_f} \pr(f)$. Therefore, if we consider the set $\desirs\coloneqq\cup_f \desirs_f$, we may say that the family of credal sets $\{\solp_f:f\in\gambles\}$ avoids sure loss (resp., is coherent) if and only if $\desirs$ does. 

Finally, note that since the credal sets are representing the local information about the desirable transactions on each gamble, the operations of marginalisation and conditioning are trivial in this case: if our assessments are given by the family of credal sets $\{\solp_f: f \in\gambles\}$, or equivalently, taking into account constant additivity, by the family $\{\solp_f: \lpr(f)=0\}$, marginalisation can be done simply by focusing on the subfamily
$\{\solp_f: f\in\gambles_{\partit}\}$ and conditioning with respect to some event $B$ means focusing on the subfamily $\{\solp_f: f=Bf\}$.

The transformation from credal sets to sets of desirable gambles in Eq.~\eqref{eq:desirs-from-credal} can be used to determine the marginal extension of marginal and conditional credal sets; the marginal extension could be computed by considering the marginal extension of the sets
\[
\desirs_{\partit}\coloneqq\cup_{f\in\gambles_{\partit}}\desirs_f
\]
and 
\[
\desirs|B\coloneqq\cup_{f\in\gambles}\{B(f-\lpr(f|B)+\eps): \eps>0\}
\]
in the manner considered in Section~\ref{sec:scale}, where $\desirs_f$ is given by Eq.~\eqref{eq:desirs-from-credal}.

\section{Preference modelling and Allais paradox}\label{sec:allais}

We conclude this paper by discussing how we can use $\kappa$-coherent sets of desirable gambles to model the preferences between two gambles $f$ and $g$. In classical decision theory \cite{anscombe1963,savage1972} preference is modelled by means of expected utility: if the reward of the alternatives depends on the outcome of some experiment taking values in $\pspace$ and the uncertainty about this experiment is modelled by means of a probability $\pr$, then under some rationality conditions it is said that $f$ is preferred to $g$, and represented $f\succ g$, if and only if $f$ dominates $g$ or $\pr(f)>\pr(g)$, which is a notion of \emph{strict preference} ($f\succ g\Leftrightarrow f-g\in\gambles^+ \text{ or } \pr(f)>\pr(g)$). More generally speaking, given a set of alternatives $J$, the \emph{optimal} ones will be those that are undominated according to the strict preference order defined above.

When there is some imprecision about the probability $\pr$, one option is to work instead with a credal set $\solp$, which may be summarised in terms of its lower and upper previsions $\lpr,\upr$. In that case, there are a few ways of generalising the notion of expected utility. In all cases, the set optimal alternatives within a set $J$ are those for which there is no other alternative that is strictly preferred to them. Depending on the manner in which this strict preference is defined, this gives rise to a number of possibilities. Specifically, given a set of alternatives $J$ we say that $f\in J$ is an optimal alternative according to: 
\begin{itemize}
    \item $\Gamma$-maximin \cite{gilboa1989} if and only if $\lpr(f)\geq \lpr(g)$ for every $g\in J$.
    \item $\Gamma$-maximax \cite{satia1973} if and only if $\upr(f)\geq\upr(g)$ for every $g\in J$.
    \item Interval dominance \cite{zaffalon2003} if and only if $\upr(f)\geq\lpr(g)$ for every $g\in J$.
    \item Maximality \cite{walley1991} if and only if $\lpr(g-f)\leq 0$ for every $g\in J$.
    \item $E$-admissibility \cite{good1952,levi1980} if and only if there is some $\pr\in\solp$ such that $\pr(f)\geq \pr(g)$ for every $g\in J$.
\end{itemize}
The intuition behind the above notions is the following: the maximin criterion compares the alternatives in terms of their worst case scenarios, while the maximax criterion takes into account only the maximum expected utility for each alternative; under interval dominance, an alternative is not optimal when its set of expected utilities is dominated by the set associated with another alternative; maximality rules out those alternatives $f$ that are worse than another alternative $g$ for all elements in the credal set; and $E$-admissibility selects those alternatives that are optimal for at least one of the models in the credal set. We refer to \cite{troffaes2007} for a more detailed discussion of these notions. 

Next we generalise these notions to the case of a $\kappa$-coherent set of desirable gambles $\desirs$. We shall take into account the correspondence between sets of desirable gambles and lower previsions in Eq.~\eqref{eq:lpr-from-desirs}. 

\begin{definition}[{\bf $\Gamma$-maximin}]
We say that $f\in J$ is optimal under the \emph{maximin} criterion if and only if
\begin{equation*}
    (\nexists g\in J, \mu \in \reals)\ g-\mu \in \desirs\text{ and }f-\mu\notin\desirs. 
\end{equation*}
\end{definition}
This criterion compares the acceptable buying prices for the gambles, and rules out those alternatives whose acceptable buying prices are also acceptable for another alternative. Note also that the use of sets of desirable gambles allows us to give an extra layer of information with respect to lower previsions: we may have for instance that $\lpr(f)=\lpr(g)$ but that $g$ is preferred to $f$, because the set of acceptable buying prices for $f$ is strictly included in those for $g$. 

\begin{definition}[{\bf $\Gamma$-maximax}]
We say that $f\in J$ is optimal under the \emph{maximax} criterion if and only if
\begin{equation*}
    (\nexists g\in J, \mu \in \reals)\ \mu-f \in \desirs\text{ and }\mu-g\notin\desirs. 
\end{equation*}
\end{definition}

The intuition here is that under the maximax criterion an alternative $f$ is ruled out if its set of acceptable selling prices is included in those of another alternative $g$. 

\begin{definition}[{\bf Interval dominance}]
We say that $f\in J$ is optimal under \emph{interval dominance} if and only if
\begin{equation*}
    (\nexists g\in J, \mu \in \reals)\ \mu-f \in \desirs\text{ and }g-\mu\in\desirs. 
\end{equation*}
\end{definition}

When $f$ is ruled out under interval dominance, there exists another alternative $g$ such that the set of acceptable buying prices for $g$ has nonempty intersection with the acceptable selling prices for $f$. 

The next two notions of optimality can only be applied when the class $\Lambda_{\kappa}$ of coherent sets is a decisive strong belief structure, i.e., when any $\kappa$-coherent set is the intersection of its decisive supersets. We first of all consider the notion of maximality: 

\begin{definition}[{\bf Maximality}]
We say that $f\in J$ is optimal in the sense of \emph{maximality} if and only if
\begin{equation*}
    (\nexists g\in J)(\forall \desirs^* \supseteq \desirs \text{ in } \tilde{\Lambda}_{\kappa})(\exists \mu \in \reals)\ g-\mu \in \desirs^*\text{ and }f-\mu\notin\desirs^*. 
\end{equation*}
\end{definition}
In other words, $f$ is optimal under maximality when there is no other alternative $g$ that is preferable to $f$ for all the complete models, which correspond to the decisive supersets of $\desirs$; and the preference of $g$ over $f$ in a decisive set of gambles is modelled by imposing that the acceptable prices for $g$ strictly include those for $f$.

The last notion is that of $E$-admissibility: 

\begin{definition}[{\bf $E$-admissibility}]
We say that $f\in J$ is optimal in the sense of \emph{$E$-admissibili\-ty} if and only if
\begin{equation*}
    (\exists \desirs^* \supseteq \desirs \text{ in } \tilde{\Lambda}_{\kappa}) (\nexists g\in J, \mu \in \reals)\ g-\mu \in \desirs^*\text{ and }f-\mu\notin\desirs^*. 
\end{equation*}
\end{definition}
The intuition here is that $f$ is $E$-admissible when there is a decisive model under which $f$ is not dominated, in the sense that no other alternative $g$ satisfies that the supremum acceptable buying price for $g$ dominates that for $f$. 

\begin{remark}
It is important to realise that preferring is different from buying in the case of nonlinear desirability. While $f\succ g$ can be characterised by one of the notions above, `buying $f$ at price $g$' (exchanging them) means that $f-g$ should be desirable.

To see that in general the two procedures are not equivalent observe that, as we shall show next, with a notion of preference it is possible to model Allais paradox in a satisfactory manner, while this is not possible when we consider the exchange between gambles.  

Nevertheless, the two notions are equivalent in some particular cases: for instance, if $\desirs$ is associated with a linear prevision $\pr$, then $\pr(f)>\pr(g)$ if and only if $\pr(f-g)>0$. Also, in the imprecise case if the gamble $g$ is constant on some $\mu\in\reals$, constant additivity implies that $\lpr(f-g)=\lpr(f-\mu)=\lpr(f)-\mu>0$ if and only if $\lpr(f)>\mu=\lpr(g)$. $\lozenge$
\end{remark}

\begin{remark}({\bf Dynamic models and St. Petersburg paradox})
One advantage of the use of nonlinear utility scale is that it allows naturally our model to change dynamically, in that it may be that $f$ is desirable after we buy $g$ but $g$ need not be desirable after we buy $f$, even if in both cases we end up accepting (or not) the gamble $f+g$ after the two steps. 

This implies in particular that, while in classical desirability we have that a gamble $f$ is preferable to a gamble $g$ if and only if $f-g$ is preferred to the status quo, this will not necessarily be the case with nonlinear desirability, as we said before. 

In order to illustrate this point, we may consider the following reformulation of the well-known St. Petersburg paradox: let $f$ be a gamble that gives us reward $1$ with probability $0.5$ and $-1$ with probability $0.5$. Consider a dynamic process where (i) initially we accept the gamble $f$; (ii) if it produces a gain, we stop the game; (iii) otherwise, we buy $2$ units of the gamble $f$; (iv) if it produces a gain, we stop the game, and otherwise we buy $2^2$ units of the gamble; and so on. In such a process, the overall gain is always positive, so the global gamble should be desirable if we assume that our closure operator satisfies monotonicity. There are however a couple of points that prevent this from being true: (a) on the one hand, it is implicitly assumed that if the gamble $f$ is desirable, so should be the gamble $2^n f$ for any natural number $n$, which need not be the case; and (b) the desirability of $g\coloneqq 2^n f$ should only be considered taking into account that this gamble only takes place when we have had a loss of $k\coloneqq2^n-1$. In other words, before deciding on the acceptability of a gamble, the set of desirable gambles $\desirs$ should be transformed taking into account the change in our status quo. $\lozenge$
\end{remark}

\subsection{Allais paradox}

Nonlinear desirability allows us to give a solution to the well-known Allais paradox \cite{allais1953}. 

Recall that in this paradox, we consider two experiments where we must choose between two gambles that pay rewards in millions of dollars. In Experiment 1, gamble $f_1$ gives a constant amount $x$ with probability 1, while gamble $f_2$ gives the same amount with probability 0.89, nothing with probability 0.01 and bigger amount $y$  with probability 0.1. 

In experiment 2 the choice is between a gamble $f_3$ that gives $x$ with probability 0.11 and nothing with probability 0.89, while $f_4$ gives $y$ with probability 0.1 and nothing with probability 0.9. 

The paradox lies in that usually players give the preference $f_1\succ f_2$ and $f_4\succ f_3$, while the difference between the gambles $f_1$ and $f_2$ should lie on whether it is preferable to win the amount $x$ with probability 0.11 or nothing with probability 0.01 and $y$ with probability 0.1, being $f_1$ and $f_2$ equal in the other cases, and the same difference holds between $f_3$ and $f_4$. In other words, if we consider an expected utility model it should be 
\[
f_1 \succ f_2 \iff f_3 \succ f_4.
\] 

In order to represent the paradox in terms of desirable gambles, we may consider a ternary space $\pspace\coloneqq\{\omega_1,\omega_2,\omega_3\}$, with the underlying assumption that $P(\{\omega_1\})=0.89$, $P(\{\omega_2\})=0.01$ and $P(\{\omega_3\})=0.1$. If we take for instance $x=1,y=1.9$ then the gambles in the experiment can be represented as: 
\begin{eqnarray*}
 f_1(\{\omega_1\})=1 & f_1(\{\omega_2\})=1 & f_1(\{\omega_3\})=1 \\
 f_2(\{\omega_1\})=1 & f_2(\{\omega_2\})=0 & f_2(\{\omega_3\})=1.9 \\
 f_3(\{\omega_1\})=0 & f_3(\{\omega_2\})=1 & f_3(\{\omega_3\})=1 \\
 f_4(\{\omega_1\})=0 & f_4(\{\omega_2\})=0 & f_4(\{\omega_3\})=1.9. \\
\end{eqnarray*}

To see that it is possible to accommodate the above preferences using nonlinear desirability, let us consider the functional\footnote{This is an instance of an \emph{ordered weighted aggregation operator (OWA)} \cite{yager1988,yager2008}; see \cite{xiong2014} for an application of these operators on different paradoxes in decision theory. The connection with OWAs also makes it simple to represent this prevision as a \emph{risk-weighted expected utility} model in the sense of Buchak \cite{buchak2013}; see \cite[Sect.~3.3]{buchak2013} for some interesting comments about Allais paradox within this theory.} 
\[
\pr(f)\coloneqq 0.4\min f +0.2 \median f+ 0.4\max f.
\]
Then $\pr$ satisfies constant additivity, is monotone and moreover $\pr(f)>0$ for every $f\in\gambles^+$ and $\pr(f)\leq 0$ for any gamble in $\gambles^-_0$. Applying Prop.~\ref{prop:desirs-pr}, the set $\desirs_{\pr}$ is $\kappa_4$-coherent. Moreover, it holds that $\desirs_{\pr}$ induces the functional $\pr$: on the one hand, 
\[
\sup\{\mu: f-\mu \in\desirs_{\pr}\}=\sup\{\mu: \pr(f-\mu)>0 \}=\pr(f),
\]
using that $\pr$ satisfies constant additivity and that $P(g)>0$ for every $g\in\gambles^+$. On the other hand, 
\[
\inf\{\mu: \mu-f \in\desirs_{\pr}\}=\inf\{\mu: \pr(\mu-f)>0 \}=\pr(f),
\]
using that 
\begin{align*}
\pr(-f)&=0.4\min (-f) +0.2 \median (-f)+ 0.4\max (-f)\\ 
&=0.4 \cdot (-\max f) +  0.2 \cdot (-\median f) + 0.4 \cdot (-\min f)\\ 
&=-0.4 \max f - 0.2 \median f - 0.4\min f=-\pr(f).
\end{align*}
We see then that $\desirs_{\pr}$ is a $\kappa_4$-coherent set of gambles that induces a precise prevision for each gamble $f$. Taking into account that 
\[
\pr(f_1)=1, \qquad \pr(f_2)=0.96, \qquad \pr(f_3)=0.6, \qquad \pr(f_4)=0.76,
\]
we deduce the preferences $f_1\succ f_2$ and $f_4\succ f_3$. 

Observe that it is not possible to solve Allais paradox with precise previsions using a closure operator that satisfies additivity (so in particular with neither $\kappa_1$ nor $\kappa_2$): the reason is that, for any real number $\mu_1<1$ and any real number $\mu_2$, if it holds that $f_2-\mu_1 \notin\desirs, f_3-\mu_2\notin\desirs$ and $f_4-\mu_2\in\desirs$, then for any $\eps>0$ it should be 
$\eps-f_2+\mu_1 \in\desirs \text{ and } \eps-f_3+\mu_2\in \desirs$, and then by additivity it would be 
\begin{equation*}
    (\eps-f_2+\mu_1)+(\eps-f_3+\mu_2)+(f_4-\mu_2)\in\desirs;
\end{equation*}
but this is the constant gamble on $2\eps+\mu_1-1$, that belongs to $\gambles_{<}$ if we pick $\eps$ small enough.

\section{Conclusions}\label{sec:conclusions}

In this paper we have introduced nonlinear desirability theory. Loosely speaking, it can be understood as a very general theory of uncertainty and value. Yet, contrary to traditional expected utility theory, probability and utility are not part of the general theory, nor can they be derived from the theory, in general. This means, for instance, that we have direct access to a subject's (lower and upper) prices for the goods under consideration, but not to assessments, such as those of probability and utility, that may have led the subject to establish those prices. In this respect the theory is somewhat more objective and more `behavioural' that standard desirability: here we are not interested in modelling a subject's inner world, but rather their actual behaviour.

In a similar spirit, the theory models a subject's attitude towards rewards that are directly expressed in amounts of goods. To make things simple we have been talking of money-valued gambles, where money is broadly intended as any (particularly nonlinear) currency. So it could be actual money or amounts of petrol, or energy, or food, etcetera. Yet we find that even talking about actual money only has a vast scope, especially in a theory like ours that aims at being very operational: for we humans have an interpersonal agreement about the value of money that can be leveraged to have transactions of essentially any sort.

To achieve this, we have abandoned the traditional schemes in decision theory, particularly those going back to Ancombe and Aumann's work. The reason is that their theory measures rewards indirectly, in a way that may well appear artificial, and because its very nature is still linear nonetheless. Savage's attempt was far more direct, but still he could not really get his theory founded on tenable assumptions---in particular because, again, of an underlying linearity assumption that leads to Allais paradox. 

Our theory has the potential to bypass all these troublesome issues; the price to pay, as we said, is losing the direct connection with probability and utility. Future work could try to explore this question in some detail, for instance by isolating the cases where that could still be done.

Most importantly, having abandoned probability and utility leads us to wonder how our theory can actually be used in practice, which in particular means: where do closure operators come from; how can a subject assess the closure operator that models their own attitudes? Recent work by Casanova et al.~\cite{casanova2023a} indicates a promising way: it gives a number of examples where attitudes to gambles in nonlinear desirability can be captured via systematic transformations of the space of gambles. These transformations, called \emph{feature mappings}, are similar in spirit to closure operators. Detailing this connection in future work would bring the present theory closer to applications.

A different while still important dimension to carefully consider if the theory is to become fully operational is the dynamic use of the models. At the moment, our theory is indeed one of desirability, in that it does not account for the evolution of one's wealth in time due to the actual buying and selling of gambles. Such a dynamical use requires a detailed analysis of desirability with regard to model revision as a consequence of changes in the status quo, and whether the closure operator can give us information about such revisions. In addition to this, one should consider that the traditional theory of desirability comes with an assumption of `act-state independence', meaning that rewards are not affected by the actions of buying and selling gambles a subject takes. If the present nonlinear generalisation of desirability is instead to become a real-world theory of decision making, situations of act-state dependence should be allowed to exist. This is to say that the present work has laid the foundations for the static case; the dynamic case entails important questions that need a dedicated future treatment to be properly addressed.

On a more theoretical level, let us recall that our theory has been developed on logical grounds, thanks to the use of closure operators. For this reason, it should contribute to clarify the relation between logic and decision theories. On the one hand, it could make it easier to verify whether or not a theory of decision making is internally consistent. On the other, it shows that logic, in quite a general sense, can be regarded as a generalised theory of desirability. This interplay between logic and desirability should lead to useful insights. For instance, since the latter comes with an embedded notion of conditioning, we might expect that logic comes with that too.

It would also be interesting to consider the problem of aggregating a number of sets of desirable gambles, generalising the ideas from \cite{casanova2021}; and it would definitely be useful and interesting to relate nonlinear desirability to Kohlas' information algebras \cite{kohlas2003}, following up on work already done in the linear case \cite{casanova2022b}. Finally, in spite of the generality of nonlinear desirability, there are dimensions of generality that are not covered by the present theory and that could be achieved by extending choice functions \cite{seidenfeld2010,vancamp2018a} to nonlinear utility scales along the lines presented here for desirability.

\section*{Acknowledgements} 
\noindent We acknowledge the financial support of project PGC2018-098623-B-I00. An early version of the ideas in this paper was presented at the 12th International Symposium on Imprecise Probabilities: Theories and Applications (ISIPTA~'21) \cite{zaffalon2021}. We would like to show our sincerest appreciation to the anonymous reviewers, whose extremely careful reading of the paper led to substantial improvements and corrections. 

\bibliographystyle{plain}

\end{document}